%% file: main-arxiv.tex
\documentclass[twoside]{article}

\usepackage[english]{babel}
\usepackage[utf8]{inputenc}
\usepackage[T1]{fontenc}
\usepackage[a4paper,left=2cm,right=2cm,top=2cm,bottom=3cm]{geometry} 
\usepackage{tikz}
\usetikzlibrary[topaths]

\usepackage{amsfonts, amsmath, bm, bbm, amsthm}  
\usepackage{natbib}

\usepackage{subfig} 

\usepackage{hyperref} 

\usepackage{float} 

\usepackage{booktabs}
\usepackage{multirow, makecell}


\usepackage{algorithm}
\usepackage[noend]{algpseudocode}

\newtheorem{theorem}{Theorem}
 
\newtheorem{definition}{Definition}
\newtheorem{lemma}{Lemma}
\newtheorem{claim}{Claim}
\newtheorem{remark}{Remark}

\newtheorem{proposition}{Proposition}

\begin{document}
% \fancyhead{}
\title{Towards Practical Lipschitz Bandits}

% \author{Anonymous}

\author{Tianyu Wang\footnote{tianyu@cs.duke.edu} \quad Weicheng Ye\footnote{weicheny@andrew.cmu.edu} \quad Dawei Geng\footnote{dawei.geng@duke.edu} \quad Cynthia Rudin\footnote{cynthia@cs.duke.edu}} 

\date{}
\maketitle

\begin{abstract}
Stochastic Lipschitz bandit algorithms balance exploration and exploitation, and have been used for a variety of important task domains.
In this paper, we present a framework for Lipschitz bandit methods that adaptively learns partitions of context- and arm-space. Due to this flexibility, the algorithm is able to efficiently optimize rewards and minimize regret, by focusing on the portions of the space that are most relevant. 
In our analysis, we link tree-based methods to Gaussian processes. In light of our analysis, we design a novel hierarchical Bayesian model for Lipschitz bandit problems. 
Our experiments show that our algorithms can achieve state-of-the-art performance in challenging real-world tasks such as neural network hyperparameter tuning.     
\end{abstract}

\input{./tex/intro-new}
\input{./tex/algorithm}
% \input{./tex/inequality} 
\input{./tex/exp-new}

\input{./tex/conclusion}

% \bibliographystyle{apa} 
% \bibliography{biblio} 

\bibliographystyle{ACM-Reference-Format}
\bibliography{biblio}

% \newpage 
% \appendix
% \input{./tex/app} 

\end{document}

%% file: tex/intro-new.tex
% !TEX root = ../writeup.tex

%\vspace{-8pt}
\section{Introduction} \label{introduction}
%\vspace{-5pt}
% \textcolor{red}{this sentence makes no sense - Stochastic bandit algorithms can be used for trading off exploration and exploitation, as well as zeroth order optimization.} 
% \textcolor{blue}{
Stochastic Lipschitz bandit algorithms are methods that balance exploration-exploitation tradeoffs. 
% }
Their usage arises in important real-world scenarios. For example, in medical trials, a doctor might deliver a sequence of treatment options with the goal of achieving the best total treatment effect, or with the goal of allocating the best treatment option as efficiently as possible, without conducting too many trials.

A stochastic bandit problem assumes that payoffs are noisy and are drawn from an unchanging distribution. The study of stochastic bandit problems started with the discrete arm setting, where the agent is faced with a finite set of choices. %, where each choice has the same cost but potentially different rewards. 
Classic works on this problem include Thompson sampling  \citep{thompson1933likelihood,agrawal2012analysis}, Gittins index \citep{gittins1979bandit}, $\epsilon$-greedy strategies \citep{sutton1998introduction}, and upper confidence bound (UCB) methods \citep{lai1985asymptotically,auer2002finite}. 
%Common solutions to this problem include the $\epsilon$-greedy algorithms \citep{sutton1998introduction}, the UCB-based algorithms \citep{auer2002finite}, and the Thompson Sampling algorithms \citep{agrawal2012analysis}. These bandit strategies have led to powerful real-life applications. For example, Deep Q-Network \citep{mnih2015human} uses $\epsilon$-greedy for action exploration; and AlphaGO \citep{silver2016mastering} uses UCT \citep{kocsis2006bandit}, which is built on the UCB strategy. %for \textcolor{red}{action searching}. %{\color{red} (UCB being a fairly complete solution  is a comment made by John Langford or Sebastien Bubeck. Well, I feel the same way on this, but still I'm not sure whether I should state that it is a relatively complete solution to this problem here - why would you do that? sounds like it's already solved, so why bother write this paper if it's solved?)}; 
One recent line of work on stochastic bandit problems considers the case where the arm space is infinite. In this setting, the arms are usually assumed to be in a subset of the Euclidean space (or a more general metric space), and the expected payoff function is assumed to be a function of the arms. 
Some works along this line model the expected payoff as a linear function of the arms \citep{auer2002using,dani2008stochastic,li2010contextual,abbasi2011improved,agrawal2013thompson}; some algorithms model the expected payoff as Gaussian processes over the arms \citep{srinivas2009gaussian,contal2014gaussian,de2012exponential}; some algorithms assume that the expected payoff is a Lipschitz function of the arms \citep{slivkins2014contextual, kleinberg2008multi,bubeck2011x, magureanu2014lipschitz}; and some assume locally H\"older payoffs on the real line \citep{auer2007improved}. When the arms are continuous and equipped with a metric, and the expected payoff is Lipschitz continuous in the arm space, we refer to the problem as a stochastic Lipschitz bandit problem. In addition, when the agent's decisions are made with the aid of contextual information, we refer to the problem as a contextual stochastic Lipschitz bandit problem. 
Not many works \cite{bubeck2011x,kleinberg2008multi,magureanu2014lipschitz} have considered the general Lipschitz bandit problem without making strong assumptions on the smoothness of rewards in context-arm space.
In this paper, we focus our study on this general (contextual) stochastic Lipschitz bandit problem, and provide practical algorithms for use in data science applications.
Specifically, we propose a framework that converts a general decision tree algorithm into an algorithm for stochastic Lipschitz bandit problems. We use a novel analysis that links our algorithms to Gaussian processes; though the underlying rewards do not need to be generated by any Gaussian process. Based on this connection, we can use a novel hierarchical Bayesian model to design a new (UCB) index. 
This new index solves two main problems suffered by partition based bandit algorithms. Namely, (1) within each bin of the partition, all arms  are treated the same; (2) disjoint bins do not use information from each other. 

Empirically, we show that using adaptively learned partitions, Lipschitz bandit algorithms can be used for hard real-world problems such as hyperparameter tuning for neural networks. 

\noindent 
\textbf{Relation to prior work: } One general way of solving stochastic Lipschitz bandit problems is to finely discretize (partition) the arm space and treat the problem as a finite-arm problem. An Upper Confidence Bound (UCB) strategy can thus be used. Previous algorithms of this kind include the \texttt{UniformMesh} algorithm \citep{kleinberg2008multi}, the HOO algorithm \citep{bubeck2011x}, and the (contextual) Zooming Bandit algorithm \citep{kleinberg2008multi, slivkins2014contextual}. While all these algorithms employ different analysis techniques, we show that as long as a discretization of the arm space fulfills certain requirements (outlined in Theorem \ref{thm:regret-bound}), these algorithms (or a possibly modified version) can be analyzed in a unified framework. 

% In previous works, from classic methods such as the Zooming bandit \citep{kleinberg2008multi, slivkins2014contextual}, to works as recent as UCRL-FA \citep{yang2019learning}, some notion of dimension is used. In particular, types of covering dimension \citep{kleinberg2008multi}, packing dimension \citep{bubeck2011x}, doubling dimension \citep{yang2019learning}, are used and are related to the regret bound. However, these dimensions, in general metric spaces, can indeed be infinity \citep{heinonen2012lectures}. When this happens, the regret bound in the above mentioned works will be meaningless (linear growth or super-linear growth). Instead, our framework does not use a notion of dimensionality, and is more general in this sense. 

The practical problem with previous methods is that they require either a fine discretization of the full arm space or restrictive control of the partition formation (e.g., Zooming rule \citep{kleinberg2008multi}), leading to implementations that are not flexible. By fitting decision trees that are grown adaptively during the run of the algorithm, our partition can be \textit{learned} from data. 
This advantage enables the algorithm to outperform leading methods for Lipschitz bandits (e.g. \cite{bubeck2011x, kleinberg2008multi}) and for zeroth order optimization (e.g. \citep{ martinez2014bayesopt, li2016hyperband}) on hard real-world problems that can involve difficult arm space and reward landscape. As shown in the experiments, in neural network hyperparameter tuning, our methods can outperform the state-of-the-art benchmark packages that are tailored for hyperparameter selection. 

In summary, our contributions are:  
%\begin{enumerate}
\textbf {1)} 
We develop a novel stochastic Lipschitz bandit framework, \textit{TreeUCB} and its contextual counterpart \textit{Contextual TreeUCB}. Our framework converts a general decision tree algorithm into a stochastic Lipschitz bandit algorithm. Algorithms arising from this framework empirically outperform benchmarks methods. 
% The gain is due to its flexibility to allow learning the reward.
\textbf{2)} We develop a new analysis framework, which can be used to recover previous known bounds, and design a new principled acquisition function in bandits and zero-th order optimization. 

%% file: tex/algorithm.tex
\section{Main results}
%\vspace{-0.3cm}
% \label{algorithm-non-contextual}
\label{sec:bandit}
\subsection{The TreeUCB framework}
%\vspace{-5pt}
%\textbf{Problem setting:} 
% The goal of a stochastic bandit algorithm is to locate the global maximum (minimum) of a payoff function in as few iterations as possible, balancing exploration and exploitation. 
%In addition, we can only access via querying the payoff function at a sequence of points. 
Stochastic bandit algorithms, in an online fashion, explore the decision space while exploit seemingly good options. 
The performance of the algorithm is typically measured by regret. In this paper, we focus our study on the following setting. A payoff function is defined over an arm space that is a compact doubling metric space $(\mathcal{A},d )$, the payoff function of interest is $f: \mathcal{A} \rightarrow [0,1]$, and the actual observations are given by $y(a) = f(a) + \epsilon_a$. In our setting, the noise distribution $\epsilon_a$ could vary with $a$, as long as it is uniformly mean zero, almost surely bounded, and independent of $f$ for every $a$. Our results easily generalize to sub-Gaussian noise \citep{shamir2011variant}.  %\footnote{PartitionUCB and Contextual PartitionUCB can naturally handle sub-Gaussian noise with zero mean.} 
In the analysis, we assume that the (expected) payoff function $f$ is Lipschitz in the sense that $\forall a, a^\prime \in \mathcal{A}$, $ \left| f (a) - f(a^\prime) \right| \le L d ( a, a^\prime ) $ for some Lipschitz constant $L$. An agent is interacting with this environment in the following fashion. At each round $t$, based on past observations $(a_1, y_1, \cdots, a_{t-1}, y_{t-1} )$, the agent makes a query at point $a_t$ and observes the (noisy) payoff $y_t$, where $y_t$ is revealed only after the agent has made a decision $a_t$. 
%To measure the performance of the agent's strategy, the concept of regret is defined. 
%The regret at round $t$ is defined to be 
% \begin{align*}
% r_t = f(a^*) - f(a_t)
% \end{align*}
% where $a^*$ is the global maximizer of $f$; and 
For an agent executing algorithm \texttt{Alg}, the regret incurred up to time $T$ is defined to be: 
\begin{align*}
R_T (\texttt{Alg}) = \sum_{t=1}^T \left(  f(a^*) - f(a_t) \right),
\end{align*}
where $a^*$ is the global maximizer of $f$. 
%\textbf{The PartitionUCB algorithm}: 

% \subsubsection*{Dimensionality General Metric Space }

% Before going to algorithm description, we take a deviate and see why bandit problems a general metric space setting has not been fully studied. 

% In previous works, as classic as Zooming bandit, or as recent reinforcement learning in metric spaces \citep{yang2019learning}, it is assumed that the underlying metric space is (geometrically) doubling. In particular, types of covering dimension \citep{kleinberg2008multi}, packing dimension \citep{bubeck2011x}, doubling dimension \citep{krizhevsky2009learning}, is used and is related to . However, these dimensions, in general metric spaces, can indeed be infinity \citep{}, and in this case, the regret bound in the above mentioned works will be meaningless (\mathcal{O} (T)). Instead, our framework do not use a notion of dimensionality, and is more general in this sense. 

% As asserted by the Assouad's Embedding Theorem (Theorem \ref{thm:embedding}), any doubling space can be bi-Lipschitzly embedded into Euclidean space with proper scaling of the metric. Therefore, the problem setting we are considering is general in this sense as well. 

Any TreeUCB algorithm runs by maintaining a sequence of finite partitions of the arm space. Intuitively, at each step $t$, TreeUCB treats the problem as a finite-arm bandit problem with respect to the partition bins at $t$, and chooses an arm uniformly at random within the chosen bin. The partition bins become smaller and smaller as the algorithm runs. Thus, at any time $t$, we maintain a partition $\mathcal{P}_t = \left\{ P_t^{(1)}, \cdots, P_t^{(k_t)} \right\}$ of the input space. That is, $ P_t^{(1)}, \cdots, P_t^{(k_t)} $ are subsets of $\mathcal{A}$, are mutually disjoint and $\cup_{i=1}^{k_t} P_t^{(i)} = \mathcal{A}$. 

As an example, Figure \ref{fig:heatmap} shows an partitioning of the input space, with the underlying reward function shown by color gradient. In an algorithm run, we collect data and estimate the reward with respect to the partition. Based on the estimate, we select a ``box'' to play next. 

% \begin{figure} 
%     \centering
%     \includegraphics[scale = 0.5]{./figures/heatmap_2}
%     \caption{Example reward function (in color gradient) with example partitioning (black and red lines). \label{fig:heatmap}} 
% \end{figure} 

\begin{figure}
% \begin{minipage}[c]{0.32\textwidth}
    \centering
    \includegraphics[width=0.5\textwidth]{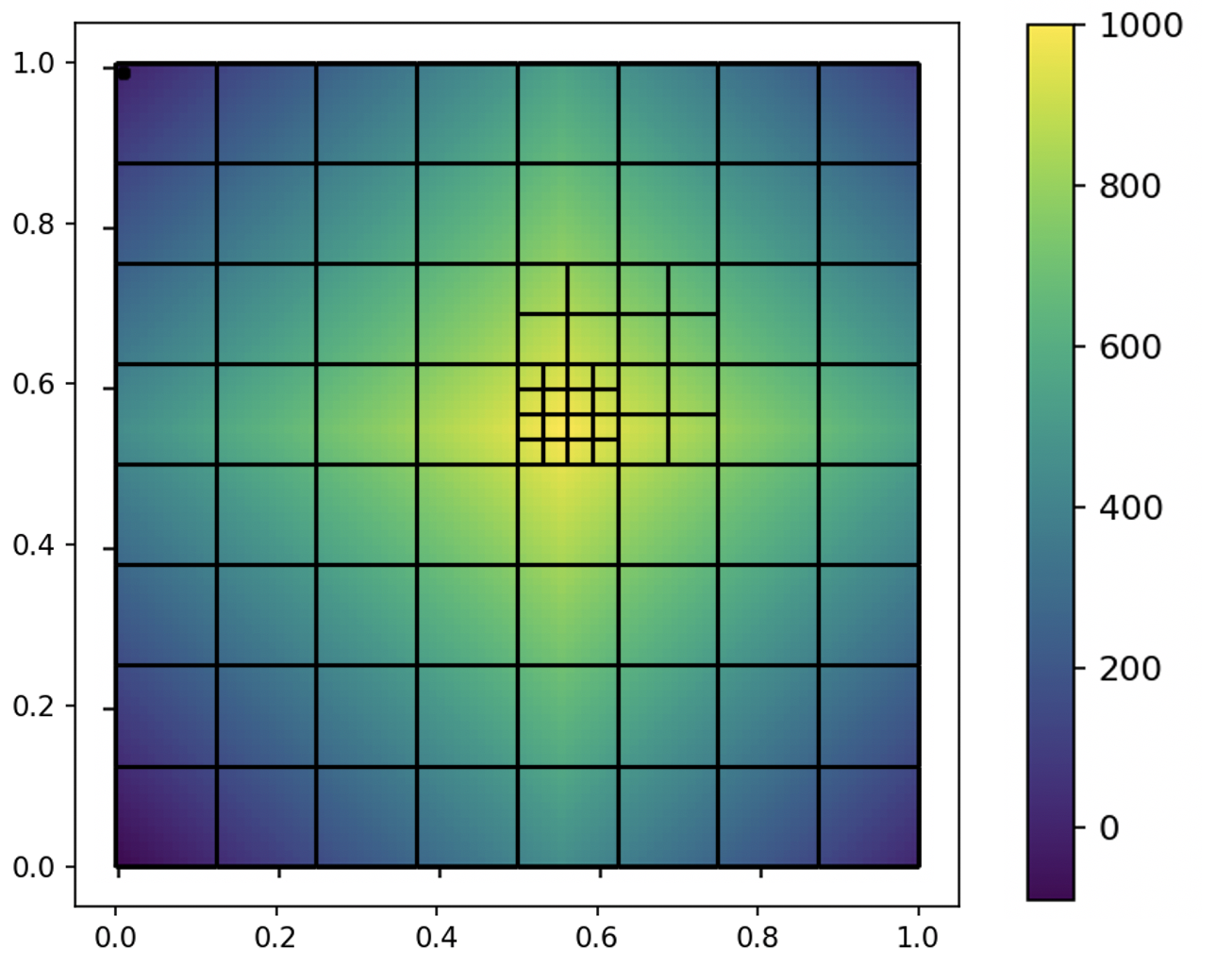}
% \end{minipage} \hfill
% \begin{minipage}[c]{0.12\textwidth}
    \caption{Example reward function (in color gradient) with an example partitioning.
    % (black and red boxes).
    \label{fig:heatmap}}
% \end{minipage}
\end{figure}
% for any $t$,
% \begin{align}
% \begin{cases}
% P_t^{(i)} \cap P_t^{(j)} = \phi, \quad \text{for } 1\le i < j \le k_t  \\
% \bigcup _{ j = 1 }^{ k_t }{ P_t^{(j)} } = \mathcal{A}. \label{eq:partition}
% \end{cases}
% \end{align}
Each element in the partition is called a \textit{region} and by convention $\mathcal{P}_0 = \{ \mathcal{A} \}$. The regions could be leaves in a tree, or chosen in some other way.

Given any $t$, if for any $P^{(i)} \in \mathcal{P}_{t+1}$, there exists $P^{(j)} \in \mathcal{P}_t$ such that $P^{(i)} \subset P^{(j)} $, we say that $\{ \mathcal{P}_t\}_{t \ge 0}$ is a sequence of \textbf{\textit{nested partitions.}} In words, at round $t$, some regions (or no regions) of the partition are split into multiple regions to form the partition at round $t+1$. We also say that the partition \textbf{\textit{grows finer.}} 

%Before formulating our strategy, we need to put forward several definitions. First of all, 
Based on the partition $\mathcal{P}_t$ at time $t$, we define an auxiliary function -- the \textit{Region Selection function}.
%-- in Definition \ref{def:region-selection} to aid our discussion.  
\begin{definition}[Region Selection Function]
\label{def:region-selection}
Given partition $\mathcal{P}_t$, function $p_t : \mathcal{A} \rightarrow \mathcal{P}_t$ is called a Region Selection Function with respect to $\mathcal{P}_t$ if for any $a \in \mathcal{A}$, $p_t(a)$ is the region in $\mathcal{P}_t$ containing $a$. 
\end{definition}
\vspace{-5pt}
%function $p_t : \mathcal{A} \rightarrow \mathcal{P}_t$ that takes an arm $a \in \mathcal{A}$ as input and returns the region that contains $a$. 
%We call $p_t$ the \textit{region selection function}. %\textcolor{red}{
% For example, if the arm space $\mathcal{A} = [0,2]$, and the partition $\mathcal{P}_t = \{ [0,1), [1,2] \}$, then $p_t$ is defined on [0,2] and
% \begin{align*}
%     p_t (a) = 
%     \begin{cases}
%         [0,1), \quad \text{if } a \in [0,1), \\
%         [1,2], \quad \text{if } a \in [1,2].
%     \end{cases}
% \end{align*}

As the name TreeUCB suggests, our framework follows an Upper Confidence Bound (UCB) strategy. In order to define our Upper Confidence Bound, we require several definitions.
%first define the \textit{count} function, the \textit{corrected average}, and the \textit{corrected count} function in Definition \ref{def:count-mean}. %Our UCB is defined by the \textit{corrected mean} and \textit{corrected count}. 
% \begin{itemize}
%     \item $)$, 
%     \item $p_t(\frac{\pi}{2}) = [0,1)$, 
%     \item $p_t(1.5) = [1,2]$.
% \end{itemize} 
% Based on the partition $\mathcal{P}_{t}$ at time $t$ ($t \ge 1$) and the observations $(a_1, y_1, a_2, y_2 , \cdots, a_{t^\prime }, y_{t^\prime } )$ received up to time $t^\prime$ ($t^\prime \ge 1$), 
% we define the functions in Definition \ref{def:count-mean}: 
\begin{definition} \label{def:count-mean}
Let $\mathcal{P}_{t}$ be the partition of $\mathcal{A}$ at time $t$ ($t \ge 1$) and let $p_t$ be the Region Selection Function associated with $\mathcal{P}_t$. Let $(a_1, y_1, a_2, y_2 , \cdots, a_{t^\prime }, y_{t^\prime } )$ be the observations received up to time $t^\prime$ ($t^\prime \ge 1$). We define \\
%\begin{itemize}
%    \vspace{-0.3cm}
%    \item \\
    $\bullet$ the \textit{count} function $n_{t, t^\prime}^0: \mathcal{A} \rightarrow \mathbb{R}$, such that 
    $$ n_{t, t^\prime}^0 (x) = \sum_{i=1}^{t^\prime } \mathbb{I}[x_i \in p_{t}(x)]. $$ 
%    \vspace{-0.3cm}
    $\bullet$ the \textit{corrected average} function $m_{t, t^\prime}: \mathcal{A} \rightarrow \mathbb{R}$, such that
    \begin{align}
    m_{t, t^\prime}(a) = 
    \begin{cases}
    \frac{\sum_{i = 1 }^{ t^\prime } y_i \mathbb{I}[a_i \in p_{t}(a)] }{ n^0_{t, t^\prime} (a) } , 
    \text{ if }  n^0_{t, t^\prime}(a) > 0; 
    \\
    1, \quad \text{otherwise}.
    \end{cases} \label{eq:corrected-average}
    \end{align} 
%    \vspace{-0.3cm}
    $\bullet$ the \textit{corrected count} function, such that
    \begin{align}
        n_{t, t^\prime} (x) &= \max \left( 1, n_{t, t^\prime}^0 (x) \right). \label{eq:corrected-count} 
    \end{align}
    When $t = t^\prime$, we shorten the notation from $m_{t,t^\prime}$ to $m_t$, $n_{t,t^\prime}^0$ to $n_t^0$, and $n_{t,t^\prime}$ to $n_t$. 
%    \vspace{-0.35cm}
%\end{itemize}%[leftmargin=*]
% \vspace{-8pt}
% \item the \textit{count} function $n^0_{t, t^\prime} : \mathcal{A} \rightarrow \mathbb{Z}$, such that
% \begin{align*}
% n^0_{t, t^\prime}(a) = \sum_{i=1}^{t^\prime } \mathbbm{1}[a_i \in p_{t}(a)] ,% & & \tag{count}
% \end{align*}
% and by convention, $n^0_{t, 0}(a) = 0$;
% \item the \textit{corrected count} function $n_{t, t^\prime} : \mathcal{A} \rightarrow \mathbb{Z}$, such that
% \begin{align}
% n_{t, t^\prime} (a) = \max \left( 1, n^0_{t, t^\prime} (a) \right) ; \label{eq:corrected-count}
% \end{align}
% \item the \textit{corrected average} function $m_{t, t^\prime}: \mathcal{A} \rightarrow \mathbb{R}$, such that
% % {\setlength{\abovedisplayskip}{2pt}%

% }%
%\begin{align}
%\end{align}
% \end{enumerate}
\end{definition}

In words, $n^0_{t,t^\prime} (a)$ is the number of points among $(a_1, a_2, \cdots, a_{t^\prime})$ that are in the same region as arm $a$, with regions as elements in $\mathcal{P}_t$. 
% It is important to note that the domain of the function $n^0_{t, t^\prime}$ is the arm space $\mathcal{A}$ -- although when computing $n^0_{t,t^\prime}(a)$ with $a \in \mathcal{A}$, we need to locate the region that contains $a$ first (using the \textit{region selection} function), the function $n^0_{t,t^\prime}$ always takes an arm $a \in \mathcal{A}$ as input. The functions $n_{t, t^\prime}$ and $m_{t, t^\prime}$ are defined in a similar fashion -- all are defined over the arm space $\mathcal{A}$, 
% with respect to the partition $\mathcal{P}_t$ and the data $(a_1, y_1, \cdots, a_{t^\prime}, y_{t^\prime})$. 
% When $t = t^\prime$, we simplify the notations $n^0_{t, t^\prime}$ to $n^0_t$, $n_{t, t^\prime}$ to $n_t$, and $m_{t, t^\prime}$ to $m_t$. 
We also denote by $D(\mathcal{S})$ the diameter of $\mathcal{S} \subset \mathcal{A}$, and $D(\mathcal{S}) := \sup_{a^\prime, a^{\prime \prime } \in \mathcal{S}} d ( a^\prime, a^{\prime \prime } ) $. 

At time $t$, based on the partition and observations, our bandit algorithm uses, for $a \in \mathcal{A}$ 
\begin{align}
U_t(a) = m_{t-1} (a) + C \sqrt{ \frac{4 \log t}{n_{t-1} (a) } } + M\cdot D (p_t (a)), \label{eq:partition-ucb}
\end{align} 
for some $C$ and $M$ as the Upper Confidence Bound of arm $a$; and we play an arm with the highest $U_t$ value (with ties broken uniformly at random). 

\begin{remark}
    As we will discuss in Section \ref{sec:hierarchical-bayesian}, the upper confidence index for our decision can take different forms other than (\ref{eq:partition-ucb}). 
\end{remark}

Here $C$ depends on the almost sure bound on the reward, and $M$ depends on the Lipschitz constant of the expected reward, which are both problem intrinsics.

Since $U_t$ is a piece-wise constant function in the arm-space and is constant within each region, playing an arm with the highest $U_t$ with random tie-breaking is equivalent to selecting the best region (under UCB) and randomly selecting an arm within the region. 
% This strategy (\ref{eq:partition-ucb}) takes a similar form to the classic UCB1 algorithm \citep{auer2002finite}. \textcolor{red}{This sentence makes your paper sound not novel - remove it or move it!}
% \textcolor{blue}{remove it.}
%Note that although the upper confidence bounds are in similar form, our algorithm should not be confused with the Upper Confidence Tree (UCT) algorithms \citep{kocsis2006bandit} and its variations that roll out simulations. 
After deciding which arm to play,
%, from the chosen region we select an action uniformly at random from the region, and 
we update the partition into a finer one if eligible.  This strategy, TreeUCB, is summarized in Algorithm \ref{alg:tucb}. We also provide a provable guarantee for TreeUCB algorithms in Theorem \ref{thm:regret-bound}.

\begin{algorithm}
    \caption{TreeUCB (TUCB)}
    \label{alg:tucb}
    \begin{algorithmic}[1] % The number tells where the line numbering should start
        \State Parameter:  $M \ge 0$ ($M \ge L$). $C > 0$. Tree fitting rule $\mathcal{R}$ that satisfies 1--4 in Theorem \ref{thm:regret-bound}. 
        
        \Statex \Comment{$C$ depends on the a.s. bound of the reward. }
        \Statex \Comment{$M$ depends on the Lipschitz constant of the expected reward. }
	   %\Comment{\footnotesize{/*$\beta_t = \Theta (\sqrt{ \log t} )$ trades off exploration and exploitation. $\eta$ is the minimal allowed gain in MAE when fitting a regression tree.*/} }
	   %\Statex \Comment{ $[z_0, z_*]$ is the resource (epochs) range allowed. }
        \For{$t = 1, 2, \dots, T$}
                 %\State Set $z_t = \exp \frac{ (t - 1) \log z_* - (N - t) \log z_0 }{ N - 1 }$. 
        		
           	\State Fit the tree $f_{t-1}$ using rule $\mathcal{R}$ on observations $( a_1, y_1, a_2, y_2, \dots, a_{t-1}, y_{t-1})$. %using tree fitting rule $\mathcal{R}$. 
%		\STATE \COMMENT{In practice, we do not have to refit the tree at every round. }
		\State With respect to the partition $\mathcal{P}_{t-1}$ defined by leaves of $f_{t-1}$, define $m_{t-1}$, $n_{t-1}$ as in (\ref{eq:corrected-average}) and (\ref{eq:corrected-count}). Play
			\begin{align}
				a_t \in \arg \max_{a \in  \mathcal{A} } \left\{ U_t ( a ) \right\}, \label{eq:ucb}
			\end{align} 
			where 
% 			$D(p_t (a_t)) := \sup_{a, a^\prime \in p_t (a_t)} d( a, a^\prime )$ is the diameter of region $p_{t-1} (a_t)$, and 
			$U_t$ is defined in (\ref{eq:partition-ucb}). 
			Ties are broken uniformly at random. 
		\State Observe the reward $y_t$. %\label{test}
           \EndFor
    \end{algorithmic}
\end{algorithm} 

% \newpage

% \vspace{-0.2cm} 
\begin{theorem} 
\label{thm:regret-bound} 
Suppose that the payoff function $f$ defined on a compact domain $ \mathcal{A}$ satisfies $f(a) \in [0, 1]$ for all $a$ and is  Lipschitz. Let $\mathcal{P}_t$ be the partition at time $t$ in Algorithm \ref{alg:tucb}. If the tree fitting rule $\mathcal{R}$ satisfies  \\
%\begin{enumerate}
    (1) $\{ \mathcal{P}_t \}_{t\ge 0}$ is a sequence of nested partitions (or the partition grows finer); \\
    (2) $|\mathcal{P}_t| = \ o (t^\gamma)$ for some $\gamma < 1$; \\
    (3) %At any $t$, 
    $D(p_t (a)) = o (1)$ for all $a \in \mathcal{A}$,  
    %$\lim_{t \rightarrow \infty} \frac{ \sum_{t=1}^T D (p_t (a_t)) }{T} = 0$,
    where 
    $$D(p_t (a)) := \sup_{a^\prime, a^{\prime \prime } \in p_t (a)} d ( a^\prime, a^{\prime \prime } )$$ is the diameter of region $p_t (a)$; \\
    (4) given all realized observations $\{(a_t, y_t)\}_{t=1}^T$, the partitions $\{ \mathcal{P}_t \}_{t=1}^T $ are deterministic;
%\end{enumerate}
then the regret for Algorithm \ref{alg:tucb} satisfies 
$$\lim_{T \rightarrow \infty }\frac{ R_T (TUCB)  }{T} = 0$$ with probability 1.
\end{theorem} 

% \textcolor{red}{was around here last time. going through the above theorem. } 
The above assumptions are all mild and reasonable. For item 1, we can use incremental tree learning \citep{utgoff1989incremental} to enforce nested partitions. For item 2, we may put a cap (that may depend on $t$) on the depth of the tree to constrain it. For item 3, we may put a cap (that may depend on $t$) on tree leaf diameters to ensure it. For item 4, any non-random tree learning rule meets this criteria, since in this case, the randomness only comes from the data (and/or number of data points observed). 

We now discuss the proof of Theorem \ref{thm:regret-bound}. Throughout the rest of the paper, we use $\widetilde{\mathcal{O}}$ to omit constants and poly-log terms unless otherwise noted. 
To prove Theorem \ref{thm:regret-bound}, we first use Claims \ref{claim:concentrate} and \ref{claim:single-step-regret} to bound the single step regret, we then use Lemma \ref{lem:point-scattering} and Assumptions (1) -- (3) to bound the total regret. 

To start with, we first present the following two claims, which may also be carefully extracted from previous works \citep[e.g.][]{bubeck2011x}. 
\begin{claim}
\label{claim:concentrate}
For an arbitrary arm $a$, and time $t$, with probability at least $1 - \frac{1}{t^4}$, we have,
\begin{align*}
    \left| m_{t-1} (a) - f(a) \right| \le L \cdot  D(p_{t-1} (a)) + C \sqrt{ \frac{ 4 \log t}{ n_{t-1} (a)} }
\end{align*} 
for a constant $C$ that depends only on the a.s. bound of the reward. 
\end{claim}

\begin{claim}
\label{claim:single-step-regret}
At any $t$,  
with probability at least $1 - \frac{1}{t^4}$, the single step regret satisfies:
\begin{align}
    f(a^*) - f(a_t) \le 2 L\cdot  D(p_{t-1} (a_t)) + 2 C \sqrt{ \frac{ 4 \log t}{ n_{t-1} (a_t)} } \label{eq:single-step-regret}
\end{align}
for a constant $C$, that depends only on the a.s. bound of the reward. 
\end{claim} 

In Section \ref{sec:ctucb}, we prove general versions of Claims \ref{claim:concentrate} and \ref{claim:single-step-regret}. 

As the tree (partition) grows finer, the term $n_{t-1} (a)$ is not necessarily increasing with $t$ (for an arbitrary fixed $a$). Therefore part of the difficulty is in bounding $ \sum_{t=1}^T \frac{1}{n_{t-1} (a_t)} $. Next, we introduce a new set of inequalities, which we call ``point scattering'' inequalities in Lemma \ref{lem:point-scattering} to bound this term. 

% \subsection*{The point scattering inequality}
\begin{lemma}[Point Scattering Inequalities]
\label{lem:point-scattering}
For an arbitrary sequence of points $a_1, a_2, \cdots$ in a space $\mathcal{A}$, and any sequence of nested partitions $\mathcal{P}_1, \mathcal{P}_2, \cdots $ of the same space $\mathcal{A}$, we have, for any $T$,
\begin{align}
    % \sum_{t=1}^T \frac{1}{n_t (x_t)} = \mathcal{O} \left( |\mathcal{P}_T| \log \left( 1 + \frac{T}{|\mathcal{P}_T|} \right) \right).  
    &\sum_{t =  1 }^{T}  \frac{1}{n_{t-1} (a_t)} \le e | \mathcal{P}_T |  \log \left( 1 + (e - 1) \frac{T }{  | \mathcal{P}_T |  } \right) ,
    \label{eq:point-scattering-gp} \\
    &\sum_{t=1}^T \frac{1}{1 + n_{t-1}^0 (a_t)} \le |\mathcal{P}_T| \left( 1 + \log \frac{T}{|\mathcal{P}_T|} \right) , \label{eq:point-scattering-1} \\
    &\sum_{t=1}^T \left( \frac{1}{ 1 + n_{t-1}^0 (a_t)} \right)^\alpha \le \frac{1}{1 - \alpha} |\mathcal{P}_T|^\alpha T^{1 - \alpha},\,\, 0 < \alpha< 1, \label{eq:point-scattering-alpha}
\end{align}
where $n_{t-1}^0$ and $n_{t-1}$ are the count and corrected count function as in Definition \ref{def:count-mean}, and $|\mathcal{P}_T|$ is the cardinality of the finite partition $\mathcal{P}_T$. 
\end{lemma}

As defined in Definition \ref{def:count-mean}, $n_{t-1}^0 (a_t)$ is the number of points that are in the same bin (in partition $\mathcal{P}_{t-1}$) as $a_t$. Also, $n_{t-1} (a_t)$ is the ``corrected'' version of $n_{t-1}^0 (a_t)$: $n_{t-1} (a_t) = \max ( 1, n_{t-1}^0 (a_t) )$. 

%To prove (\ref{eq:point-scattering-1}) and (\ref{eq:point-scattering-alpha}), we use a relabelling trick. 
%See Appendix \ref{app:point-scattering-1} and \ref{app:point-scattering-alpha} for details.

\begin{remark} \label{remark}
    We shall notice that (\ref{eq:point-scattering-gp}) allows us to somewhat ``look one step ahead of time'', since it uses the values $\{ n_{t-1} (a_t) \}_{t}$ - the corrected counts without including $a_t$. This is because $n_{t-1}$ is computed using points up to time $t-1$. The equation (\ref{eq:point-scattering-1}) is different from (\ref{eq:point-scattering-gp}) in the sense that $\{ 1 + n_{t-1}^0 (a_t)\}_{t}$ are essentially the counts including $a_t$. While, with proper modification, both  (\ref{eq:point-scattering-gp}) and  (\ref{eq:point-scattering-1}) can be used to derive Theorem \ref{thm:regret-bound}, we shall not ignore the difference between (\ref{eq:point-scattering-gp}) and  (\ref{eq:point-scattering-1}). 
    % In fact, there are cases where (\ref{eq:point-scattering-gp}) needs to be used and gives a sharper bound than otherwise. 
\end{remark}

% To prove (\ref{eq:point-scattering-gp}), we construct hypothetical noisy Gaussian processes and connect the quantity on the left hand side to a information quantity in the Gaussian processes. 

\input{./tex/inequality.tex}

%\vspace{-0.48cm}

\subsubsection*{Proof of Theorem \ref{thm:regret-bound}}
%\vspace{-0.2cm}
Now we are ready to prove Theorem \ref{thm:regret-bound}. We can split the sum of regrets by 
$$\sum_{t=1}^T \left( f (a^*) - f(a_t) \right) = \sum_{t=1}^{\left\lfloor \sqrt{T} \right\rfloor} \left( f (a^*) - f(a_t) \right) + \sum_{\left\lfloor \sqrt{T} \right\rfloor + 1}^T \left( f (a^*) - f(a_t) \right).$$
Also, by Claim \ref{claim:single-step-regret}, with probability at least $1 - \frac{1}{3\left\lfloor \sqrt{T} \right\rfloor ^ 3} $, (\ref{eq:single-step-regret}) holds simultaneously for all $t = \left\lfloor \sqrt{T} \right\rfloor + 1, \cdots, T$ ($T \ge 2$). Thus for $T \ge 2$, the event 
\begin{align}
    E_T &= \left\{ \frac{R_T}{T} > \frac{1}{T} \left( \sqrt{T} + \sum_{t = \left\lfloor \sqrt{T} \right\rfloor + 1}^T B_t \right) \right\}, \quad \nonumber where \\
    B_t &:= \left( 2 L\cdot  D(p_{t-1} (a_t)) + 2 C \sqrt{ \frac{ 4 \log t}{ n_{t-1} (a_t)} } \right) \nonumber 
\end{align}
occurs with probability at most $\frac{1}{3\left\lfloor \sqrt{T} \right\rfloor ^ 3 }$. Since $ \frac{1}{3\left\lfloor \sqrt{T} \right\rfloor ^ 3 } \sim \frac{1}{3T^{3/2}} $, we know $\sum_{T=2}^\infty  \mathbb{P} (E_T) < \infty$. By the Borel-Cantelli lemma, we know $\mathbb{P}  \left(\lim \sup_{T \rightarrow \infty} E_T \right) = 0 $. In other words, with probability 1, $E_T$ occurs finitely many times. Thus, with probability 1, there exists a constant $T_0$, such that the event $ \overline{E}_T$ (negation of $E_T$) occurs for all $T>T_0$. Also, from the Cauchy-Schwarz inequality (used below in the second line) and (\ref{eq:point-scattering-gp}) (used below in the last line), we know that 
\begin{align*} 
& \sum_{t= \left\lfloor \sqrt{T} \right\rfloor + 1}^T \sqrt{ \frac{  \log t}{ n_{t-1} (a_t)} } \le \sum_{t = 1}^T \sqrt{ \frac{  \log t}{ n_{t-1} (a_t)} } \le \sqrt{T \log T} \sqrt{ \sum_{t=1}^T \frac{ 1 }{ n_{t-1} (a_t)} } \\
&\le \sqrt{T \log T} \sqrt{ e | \mathcal{P}_T |  \log \left( 1 + (e - 1) \frac{T }{  | \mathcal{P}_T |  } \right) } = \widetilde{\mathcal{O}} \left( T^{\frac{1 + \gamma}{2}} \right) ,
\end{align*}
where the last equality is from the assumption that $|\mathcal{P}_t| = o(t^\gamma)$ for some $\gamma < 1$. 
% Here $\tilde{\mathcal{O}}$ hides poly-log terms. 
This means 
$$
\lim_{T \rightarrow \infty}\frac{1}{T} \sum_{t=1}^T \sqrt{ \frac{4 \log t}{n_{t-1} (a_t)} } = 0 .
$$

In addition, by the assumption that $D(p_{t} (a)) = o(1)$, we know $\lim \sup_{T \rightarrow \infty} \frac{1}{T} \sum_{t=1}^T D(p_{t-1} (a)) = 0$. The above two limits give us 
%\vspace{-0.4cm}
\begin{align}
    &\lim_{T \rightarrow \infty} \frac{1}{T} \left( \sqrt{T} + \sum_{\left\lfloor \sqrt{T} \right\rfloor + 1}^T B_t \right) = 0, \quad where \label{eq:limit} \\
    &B_t := \left( 2 L\cdot  D(p_{t-1} (a_t)) + 2 C \sqrt{ \frac{ 4 \log t}{ n_{t-1} (a_t)} } \right).
\end{align}
% \vspace{-0.2cm}
Combining all the facts above, we have
% we know with probability 1,  $\lim_{T \rightarrow \infty} \frac{R_T}{T}$ is upper bounded by (\ref{eq:limit}), which means 
$
\lim_{T \rightarrow \infty} \frac{R_T}{T} = 0
$ 
with probability 1.

\textbf{Adaptive partitioning}: TUCB shall be implemented using regression trees or incremental regression trees. This naturally leverages the practical advantages of regression trees. Leaves in a regression tree form a partition of the space. Also, a regression tree is designed to fit an underlying function. This leads to an adaptive partitioning where the underlying function values within each region should be relatively similar to each other. We defer the discussion on the implementation we use in our experiments to Section \ref{sec:exp}. Please refer to \citep{breiman1984classification} for more details about regression tree fitting.
% \footnote{We provide two implementations. One uses the \texttt{scikit-learn} package \citep{scikit-learn}, the other implements an incremental tree to ensure nested partition. The code and installation instructions are in the supplementary materials.} Next, we discuss the contextual counterpart of the TreeUCB algorithm. 

% Next, we prove Lemma \ref{lem:point-scattering} and use it to reproduce some known regret bound for two other stochastic bandit algorithms. 
%\vspace{-8pt}
\subsection{The Contextual TreeUCB algorithm}
%\vspace{-5pt}
\label{sec:ctucb}
\begin{algorithm}[ht!]
    \caption{Contextual TreeUCB (CTUCB)} 
    \label{alg:ctucb}
    \begin{algorithmic}[1] % The number tells where the line numbering should start
        	   \State Parameter: $M > 0$, $C > 0$, and tree fitting rule $\mathcal{R}$. 
        \For{$t = 1, 2, \dots, T$}
                 \State Observe context $z_t$. 
           	\State Fit a regression tree $f_{t-1}$ (using rule $\mathcal{R}$) on observations $\{ (z_t, a_t), y_t \}_{t=1}^T$, .  
		\State With respect to the partition $\mathcal{P}_{t-1}$ defined by leaves of $f_{t-1}$, define $ m_{t-1}$ and $n_{t-1}$ in  (\ref{eq:corrected-average}) and (\ref{eq:corrected-count}) (over the joint space $\mathcal{Z} \times \mathcal{A}$). Play 
% 			\begin{align*}
% 				a_t \in \arg \max_{a \in \mathcal{A} } \left[ m_{t-1}(z_t, a) + \beta_t \sqrt{ \frac{1}{n_{t-1} (z_t, a)} } \right].
% 			\end{align*}
			\begin{align*}
				a_t \in \arg \max_{a \in  \mathcal{A} } \left\{ U_t ( (z_t, a) ) \right\},
			\end{align*} 
			where $U_t (\cdot)$ is defined in (\ref{eq:partition-ucb}).
			Ties are broken at random. 
		\State Observe the reward $y_t$. 
           \EndFor
    \end{algorithmic}
    % \vspace{-8pt}
\end{algorithm}
In this section, we present an extension of Algorithm \ref{alg:tucb} for the  contextual stochastic  bandit problem. 
The  contextual stochastic  bandit problem is an extension to the stochastic bandit problem. In this problem, at each time, context information is revealed, and the agent chooses an arm based on past experience as well as the contextual information. Formally, the expected payoff function $f$ is defined over the product of the context space $\mathcal{Z}$ and the arm space $\mathcal{A}$ and takes values from $[0,1]$. Similar to the previous discussions, compactness of the product space and Lipschitzness of the payoff function are assumed. In addition, a mean zero, almost surely bounded noise that is independent of the expected reward function is added to the observed rewards. %We assume that $\mathcal{Z} \times \mathcal{A}$ is compact and $\mathcal{Z} \times \mathcal{A} \subset \mathbb{R}^d$. The payoff function is assumed to be Lipschitz and $f: \mathcal{Z} \times \mathcal{A} \rightarrow [0,1]$. 
At each time $t$, a contextual vector $z_t \in \mathcal{Z}$ is revealed and the agent plays an arm $a_t \in \mathcal{A}$. The performance of the agent following algorithm \texttt{Alg} is measured by the cumulative contextual regret 
\begin{align}
R_T^c (\texttt{Alg}) = \sum_{t=1}^T f ( z_t, a_t^* ) - f (z_t, a_t), 
\end{align}
where $f ( z_t, a_t^* )$ is the maximal value of $f$ given contextual information $z_t$. %Here, $a^*_t$ is the maximizer of $f ( z_t, \cdot ).$ % and it depends on $z_t$. 
%Let $\mathcal{X} = \mathcal{Z} \times \mathcal{A}$ and $x = (z, a)$. 
%We can repeat the derivations in Section \ref{algorithm-non-contextual} on the joint space $\mathcal{X}$ instead of $\mathcal{A}$ to derive an extension of Algorithm \ref{alg:bandit}. 
A simple extension of Algorithm \ref{alg:tucb} can solve the contextual version problem. 
%In order to be, we repeat some definitions for the . 
%The readers can skip to the Pseudocode for Contextual PartitionUCB in \ref{alg:contextual-bandit}. 
%As an analog to Algorithm \ref{alg:bandit}, we develop the Contextual PartitionUCB algorithm as summarized in Algorithm \ref{alg:contextual-bandit}. 
\textit{In particular, in the contextual case, we partition the joint space $\mathcal{Z} \times \mathcal{A}$ instead of the arm space $\mathcal{A}$.} As an analog to (\ref{eq:corrected-count}) and (\ref{eq:corrected-average}), we define the corrected count $n_t$ and the corrected average $m_t$ over the joint space $\mathcal{Z} \times \mathcal{A}$ with respect to the partition $\mathcal{P}_t$ of the joint space $\mathcal{Z} \times \mathcal{A}$, and observations in the joint space $( (z_1, a_1), y_1, \cdots, (z_t, a_t), y_t )$. The guarantee of Algorithm \ref{alg:ctucb} is in Theorem \ref{thm:regret-bound-contextual}. 

\begin{theorem}
\label{thm:regret-bound-contextual} 
Suppose that the payoff function $f$ defined on a compact doubling metric space $(\mathcal{Z} \times \mathcal{A}, d)$ satisfies $f(z, a) \in [0, 1]$ for all $(z, a)$ and is Lipschitz. If the tree growing rule satisfies requirements 1-4 listed in Theorem \ref{thm:regret-bound}, then $\lim_{T \rightarrow \infty} \frac{  R_T^c (CTUCB) }{T} = 0 $ with probability 1. 
\end{theorem} 

Theorem \ref{thm:regret-bound-contextual} follows from Theorem \ref{thm:regret-bound}. Since the point scattering inequality holds for any sequence of (context-)arms, we can replace regret with contextual regret and alter Claims  \ref{claim:concentrate} and \ref{claim:single-step-regret} accordingly to prove Theorem \ref{thm:regret-bound-contextual}. 

In particular, Claims \ref{claim:concentrate} and \ref{claim:single-step-regret} extend to the contextual setting, as stated and proved below. 

\begin{claim}
    \label{claim:concentration-contextual}
    For any context $z$, arm $a$, and time $t$, with probability at most $\frac{1}{t^4}$, we have:
\begin{align}
    & \left| m_{t-1} (z, a) - f(z, a) \right| > L \cdot  D(p_{t-1} (z, a)) + C \sqrt{ \frac{ 4 \log t}{ n_{t-1} (z, a)} } \label{eq:claim3}
\end{align} 
for a constant $C$. 
\end{claim}

\begin{proof}
    First of all, when $t = 1$, this is trivially true by Lipschitzness. Now let us consider the case when $t \ge 2$. 
    Let us use $A_1, A_2, \cdots, A_t$ to denote the random variables of arms selected up to time $t$, $Z_1, Z_2, \cdots, Z_t$ to denote the random context up to time $t$ and $Y_1, Y_2, \cdots, Y_t$ to denote random variables of rewards received up to time $t$. 
    Then the random variables $ \left\{ \sum_{t=1}^T \left( f (Z_t, A_t) - Y_t \right) \right\} $ is a martingale sequence.
    % with respect to the filtration $\sigma ( A_1, Y_1, \cdots, A_t, Y_t )$. 
    This is easy to verify since the noise is mean zero and independent. In addition, since there is no randomness in the partition formation (given a sequence of observations), for a fixed $a$, we have the times  $\mathbb{I} [(Z_t, A_i) \in p_{t-1} (z, a) ] $ ($i \le t$) is measureable with respect to $\sigma (Z_1, A_1, Y_1, \cdots, Z_t, A_t, Y_t)$. Therefore, the sequence $ \left\{ \sum_{i=1}^t \left( f(Z_i, A_i) - Y_i \right) \mathbb{I} [ (Z_i, A_i) \in p_{t-1} (z, a) ] \right\}_{t = 1}^T $ is a skipped martingale. Since skipped martingale is also a martingale, we apply the  Azuma-Hoeffding inequality (with sub-Gaussian tails) \citep{shamir2011variant}. For simplicity, we write 
    \begin{align}
        B_t (z,a) &:= C   \sqrt{ \frac{ 4 \log t}{ n_{t-1} (z, a) } } + L\cdot D(p_{t-1} (z, a)), \\
        \mathcal{E}_t^i (z,a) &:=   (Z_i, A_i) \in p_{t-1} (z, a). 
    \end{align}
    Combining this with Lipschitzness, we get there is a constant $C$ (depends on the a.s. bound of the reward, as a result of Hoeffding inequality), such that 
    \begin{align}
        &\quad \mathbb{P} \left\{ \left| m_{t-1} (z, a) - f (z, a) \right| > B_t (z,a) \right\} \nonumber \\ 
        &\le \mathbb{P} \left\{ \left| \frac{ 1}{ n_{t-1} (z, a) }
        \sum_{i=1}^{t-1} \left( f(Z_i, A_i) - Y_i \right) \mathbb{I} [ \mathcal{E}_t^i (z,a) ] \right| \right. \nonumber \\
        &\quad \left.+ \left| f(z, a) - \frac{ 1 }{ n_{t-1} (z, a ) }  
        \sum_{i=1}^{t-1}  f(Z_i, A_i)  \mathbb{I} [ \mathcal{E}_t^i (z,a) ] \right| \right. \nonumber \\
        &\quad > \left. C  \sqrt{ \frac{ 4 \log t}{ n_{t-1} (z, a) } } + L\cdot D(p_{t-1} (z, a)) \right\} \label{eq:lipschitz-concentration} \le \frac{1}{t^4},  
    \end{align}
    where (\ref{eq:lipschitz-concentration}) uses both the Lipschitzness and the Azuma-Hoeffding's inequality. 
\end{proof}

\begin{claim}
\label{claim:single-step-regret-contextual}
At any $t$,  
with probability at least $1 - \frac{1}{t^4}$, the single step contextual regret satisfies: 
\begin{align*}
     f(z_t, a_t^*) - f(z_t, a_t) \le 2 L\cdot  D(p_{t-1} (z_t, a_t)) + 2 C \sqrt{ \frac{ 4 \log t}{ n_{t-1} (z_t, a_t)} }
\end{align*}
for a constant $C$. Here $a_t^*$ is the optimal arm for the context $z_t$. 
\end{claim}

\begin{proof}
    By Claim \ref{claim:concentration-contextual}, with probability at least $1 - \frac{1}{t^4}$, the following ((\ref{eq:eq-dummy1}) and (\ref{eq:eq-dummy2})) hold simultaneously, 
    \begin{align}
        &\quad m_{t-1} (z_t, a_t) + C \sqrt{ \frac{ 4 \log t }{ n_{t-1} (z_t, a_t) } } + L\cdot D(p_{t-1} (z_t, a_t)) \nonumber  \\
        &\ge m_{t-1} (z_t, a_t^*) + \sqrt{ \frac{ 4 \log t }{ n_{t-1} (z_t, a_t^*) } } + L\cdot D(p_{t-1} (z_t, a_t^*)) \nonumber  \\
        &\ge f(z_t, a_t^*),
        \label{eq:eq-dummy1}
    \end{align}
    \begin{align}
        f(z_t, a_t) &\ge m_{t-1} (z_t, a_t) - C \sqrt{ \frac{ 4 \log t }{ n_{t-1} (z_t, a_t) } }  - L\cdot D(p_{t-1} (z_t, a_t)) . \label{eq:eq-dummy2}
    \end{align}
    This is true since we first take a one-sided version of Hoeffding-type tail bound in (\ref{eq:claim3}), and then take a union bound over the two points $ (z_t,a_t) $ and $ (z_t,a_t^*) $. This first halves the probability bound and then doubles it. Then we take the complementary event to get (\ref{eq:eq-dummy1}) and (\ref{eq:eq-dummy2}) simultaneously hold with probability at least $1 - \frac{1}{t^4}$. We then take another union bound over time $t$, as discussed in the main text. Note that throughout the proof, we do not need to take union bounds over all arms or all regions in the partition. 
    
    Equation \ref{eq:eq-dummy1} holds by algorithm definition. Otherwise we will not select $a_t$ at time $t$. 
    Combine (\ref{eq:eq-dummy1}) and (\ref{eq:eq-dummy2}), and we get 
    \begin{align*}
        &\quad f (z_t, a_t^*) - f(z_t, a_t) \\
        &= f (z_t, a_t^*) - m_{t-1} (z_t, a_t)  + m_{t-1} (z_t, a_t) - f(z_t, a_t) \\
        &\le 2 C \sqrt{ \frac{ 4 \log t}{n_{t-1} (z_t, a_t)} }  + 2 L\cdot D(p_{t-1} (z_t, a_t)). 
    \end{align*} 
\end{proof}

%(See Appendix \ref{app:ctucb}) %This demonstrates that our analysis

\input{./tex/table-svhn}

\input{./tex/table-cifar}
% Detailed proofs of Lemmas \ref{lem:truth-mean-difference-whp} and  \ref{lem:single-regret-bound-whp} are in Appendix \ref{app:proofs}.
%\vspace{-0.4cm}
\subsection{Use Cases of Point Scattering Inequalities} \label{sec:use-point-scattering}

\subsubsection{Recover Previous Bounds}
%\vspace{-0.3cm}
In this section, we give examples of using the point scattering inequalities to derive regret bounds for other algorithms. For our purpose of illustrating the point scattering inequalities, the discussed algorithms are simplified. We also assume that the reward and the sub-Gaussianity are properly scaled so that the parameter before the Hoeffding-type concentration term is $1$. 
%\vspace{-0.3cm}

\textbf{The UCB1 algorithm} 
%\vspace{-0.3cm}
The classic UCB1 algorithm \citep{auer2002finite} assumes a finite set of arms, each having a different reward distribution. Following our notation, at time $t$, the UCB1 algorithm plays 
\begin{align} 
    a_t \in \arg \max_a 
    \left\{ m_{t-1} (a) + \sqrt{ \frac{2 \log T}{n_{t-1} (a)} } \right\} . \label{eq:ucb1}
\end{align} 
Indeed, this equation can be interpreted as (\ref{eq:ucb}) under the discrete 0-1 metric: two points are distance zero if they coincide and distance 1 otherwise. 
% Thus the Lipschitzness term vanishes in this case and Claims \ref{claim:concentrate} and \ref{claim:single-step-regret} can be modified appropriately (see Appendix \ref{app:ucb1}). 
Then from the point scattering inequality (\ref{eq:point-scattering-gp}), we get for UCB1 
% \vspace{-0.5cm}
\begin{align*}
    \mathbb{E} [R_T (UCB1)] &
    % \mathcal{O} \left( \sum_{t=1}^T \sqrt{ \frac{ \log t}{n_t (a_t)} } \right) 
    = \mathcal{O} \left( \sum_{t=1}^T  \sqrt{ \frac{ \log T }{n_{t-1} (a_t)} } \right) 
    \\
    &= \mathcal{O} \left( \sqrt{T \log T}  \sqrt{\sum_{t=1}^T \frac{ 1 }{n_{t-1} (a_t)} } \right) = \widetilde{\mathcal{O}} \left(  \sqrt{ K\cdot T }  \right), 
\end{align*}
where $K$ is number of arms in the problem. 
% This matches the gap-independent bound derived using traditional methods up to a (sub-)log factor. \textcolor{red}{previous sentence is not clear. What is gap? What are the previous works?} 
% \textcolor{red}
% {
This matches the gap-independent (independent of the reward gap between an arm and the optimal arm) bound derived using traditional methods in UCB1 algorithm \citep{auer2002finite,bubeck2012regret}.
In this analysis, we apply the point scattering inequality with the partition $\mathcal{P}_t$ being the set of arms at all $t$. 
%See Appendix \ref{app:ucb1} for details. 

%\vspace{-0.35cm}

\textbf{Finite Time Bound for Lipschitz Bandits and Lipschitz RL. } As shown in Claim \ref{claim:single-step-regret}, the single step regret is bounded by a Hoeffding-type concentration and the diameter of selected region (due to Lipschitzness). Since the point scattering inequalities provide a bound of the overall summation of the Hoeffding terms, we can design and analyze many partition-based Lipschitz algorithms using point scattering inequalities We can do this since the partitioning is up to our choice. Examples include the \texttt{UniformMesh} algorithm discussed by \citep{kleinberg2008multi}, and parition-based Lipschitz reinforcement learning algorithm recently studied \citep[e.g.][]{yang2019learning}.

\input{./tex/robust}

%% file: tex/inequality.tex
% \vspace{-15pt}
\subsubsection{Proof of (\ref{eq:point-scattering-gp})}
% \vspace{-5pt}
\label{app:point-scattering-gp}
We use a novel constructive trick to derive (\ref{eq:point-scattering-gp}). This trick and the usefulness of the result (Remarks \ref{remark} and \ref{remark2} and Section \ref{sec:use-point-scattering}) mark our major technical contribution. The trick is to consider the incidence matrix of which points are within the same partition bin, and use this matrix as if it were a covariance matrix for a Gaussian process. Then, we use knowledge about Gaussian processes to bound the sum of the inverse of the number of points in each bin over time.
% Considering the usefulness of (\ref{eq:point-scattering-gp}) (Remark \ref{remark}), it marks our major technical contribution. 

For each $T$, we construct a hypothetical noisy degenerate Gaussian process. We are not assuming our payoffs are drawn from these Gaussian processes. We only use these Gaussian processes as a proof tool.
%use the information gain in the constructed Gaussian processes to bound the term $\sum_{t = \left\lfloor \sqrt{T} \right\rfloor + 1 }^T \frac{ 1 }{\sqrt{n_{t-1} (a_t)}} $. 
To construct these noisy degenerate Gaussian processes, we define the kernel functions $k_T: \mathcal{A} \times \mathcal{A} \rightarrow \mathbb{R}$, 
\begin{align}
k_T (a, a^\prime) = \begin{cases} 1, \quad \text{if } p_T (a) = p_T (a^\prime)  \\ 0, \quad \text{otherwise.}  \label{eq:tree-kernel}  \end{cases} 
\end{align}
where $p_T$ is the region selection function defined with respect to $\mathcal{P}_T$. The kernel $k_T$ is positive semi-definite as shown in Proposition \ref{prop}.
\begin{proposition} \label{prop}
The kernel defined in (\ref{eq:tree-kernel}) is positive semi-definite for any $T \ge 1$. 
%(Proof in Appendix \ref{app:proof-prop})
\end{proposition}

\begin{proof} 
For any $x_1, \dots, x_n$ in where the kernel $k_T (\cdot, \cdot)$ is defined, the Gram matrix $ K = \begin{bmatrix}  k_T (x_i, x_j) \end{bmatrix}_{n \times n} $ can be written into block diagonal form where diagonal blocks are all-one matrices and off-diagonal blocks are all zeros with proper permutations of rows and columns. Thus without loss of generality, for any vector $\bm{v} = [v_1, v_2, \dots, v_n] \in \mathbb{R}^n$, $\bm{v}^\top K \bm{v} = \sum_{b = 1}^B \left( \sum_{j:i_j \text{ in block }b } v_{i_j} \right)^2 \ge 0$ where the first summation is taken over all diagonal blocks and $B$ is the total number of diagonal blocks in the Gram matrix.
\end{proof}

%\vspace{-8pt}
% \vspace{-4pt}
Now, at any time $T$, let us consider the model $\tilde{y}(a) = g(a) + e_T$ where $g$ is drawn from a Gaussian process $g \sim \mathcal{GP} \left( 0 , k_T (\cdot, \cdot) \right)$ and $e_T \sim \mathcal{N} (0, s^2_T)$. Suppose that the arms and hypothetical payoffs $\{(a_1, \tilde{y}_1), (a_2, \tilde{y}_2), \dots, (a_t, \tilde{y}_t)\}$  are observed from this Gaussian process. The posterior variance for this Gaussian process after the observations at $a_1, a_2, \dots, a_t$ is 
\begin{align*} 
\sigma^2_{T, t} (a) = k_T (a,a) - \bm{k}_a^T (K + s^2_T I)^{-1} \bm{k}_a ,
\end{align*} 
where $\bm{k}_a = [ k_T (a, a_1), \dots, k_T (a, a_t) ]^\top$, $K = [ k_T (a_i, a_j) ]_{t \times t}$ and $I$ is the identity matrix. In other words,  $\sigma^2_{T,t} (a)$ is the posterior variance using points up to time $t$ with the kernel defined by the partition at time $T$. %Similarly, for the partition at time $T$, we use $n_{T,t}(a)$ to denote the number of samples among $x_1, x_2, \cdots, x_t$ that fall into the region that contains $x$. 
After some matrix manipulation, we know that %for a Gaussian process with this kernel, the posterior variance at time $t$ is 
\begin{align*} 
\sigma^2_{T,t} (a) = 1 - \bm{1}_a [ \bm{1}_a \bm{1}_a^\top + s^2_T I ]^{-1} \bm{1}_a,
\end{align*} 
where $\bm{1}_a = [1,\cdots,1]_{1 \times {n^0_{T,t} (a)} }^\top$.
By the Sherman-Morrison formula, $[ \bm{1}_a \bm{1}_a^\top + s^2_T I ]^{-1} = s^{-2}_T I - \frac{ s^{-4}_T \bm{1}_a \bm{1}_a^\top }{ 1 + s^{-2}_T n^0_{T,t} (a) }$. Thus the posterior variance is 
\begin{align}
\sigma^2_{T,t} (a) = \frac{1}{1 + s^{-2}_T n^0_{T,t} (a)}. \label{eq:n-to-sigma}
\end{align}
%Now, we can link the sum of variances in the constructed Gaussian processes to (\ref{eq:sum-inv-count}), 
% since the posterior variances $\sigma^2_{T,t} (a)$ in these Gaussian processes can be used to bound the term $ \frac{ 1 }{\sqrt{n_{t-1} (a)}} $. 

Following the arguments in \citep{srinivas2009gaussian}, we derive the following results. For any $t \le T$, and an arbitrary sequence $\bm{a}_t = \{ a_1, a_2, \cdots, a_t \}$, we consider fixing this sequence and query the constructed Gaussian processes at these points. Since $\bm{a}_t$ is fixed, the entropy $H (\tilde{\bm{y}}_t , \bm{a}_t ) = H ( \tilde{\bm{y}}_t )  $. Since, by definition of a Gaussian process, $\tilde{\bm{y}}_t$ follows a multivariate Gaussian distribution, 
\begin{align} 
H (\tilde{\bm{y}}_t ) = \frac{1}{2} \log \left[ (2 \pi e )^t \det \left( K + s_T^2 I \right) \right]  \label{eq:entropy-LHS} 
\end{align} 
where $K = \begin{bmatrix}  k_T (a_i, a_j) \end{bmatrix}_{t \times t}$. We can then compute $H ( \tilde{\bm{y}}_t)$ by
\begin{align} 
H ( \tilde{\bm{y}}_t) &= H( \tilde{y}_t | \tilde{ \bm{y} }_{t-1} ) + H( \tilde{ \bm{y} }_{t-1} ) \nonumber \\ 
&= H( \tilde{y}_t | a_t, \tilde{ \bm{y} }_{t-1}, \bm{a}_{t-1} ) + H( \tilde{ \bm{y} }_{t-1} ) \nonumber \\ 
&=  \frac{1}{2} \log \left( 2 \pi e \left( s_T^2 + \sigma_{T,t-1}^2 (a_t) \right) \right) + H( \tilde{ \bm{y} }_{t-1} ) \nonumber \\ 
&= \frac{1}{2} \sum_{\tau =1}^t \log \left( 2 \pi e \left( s_T^2 + \sigma_{T, \tau -1}^2 (a_\tau ) \right) \right), \label{eq:entropy-RHS} 
\end{align} 
where (\ref{eq:entropy-RHS}) comes from recursively expanding $H(\tilde{\bm{y}}_\tau)$. By (\ref{eq:entropy-LHS}) and (\ref{eq:entropy-RHS}), 
\begin{align} 
\sum_{\tau=1}^t \log \left( 1 + s^{-2} \sigma_{ T, \tau -1}^2 (a_\tau ) \right)  = \log \left[ \det \left( s^{-2} K + I \right) \right]. \label{eq:recursive-equation}
\end{align} 
% \textcolor{red}{was around here. } 
For the block diagonal matrix $K$ of size $t \times t$, let $n_i$ denote the size of block $i$ and $B^\prime$ ($ B^\prime \le |\mathcal{P}_t| $) be the total number of diagonal blocks up to a time $t$ ($t \le T$). Then we have 
\begin{align}
\det \left( s^{-2} K + I \right) &= \prod_{i = 1}^{B^\prime } \det \left( s^{-2} \bm{1} \bm{1}^\top + I_{n_i \times n_i} \right) \nonumber \\
&= \prod_{i = 1}^{B^\prime }  \left( 1 + s^{-2} n_i \right) \le \left( 1 + \frac{s^{-2} t}{B^\prime} \right)^{B^\prime} \nonumber ,
\end{align}
where $\bm{1}$ is all-1 vector of proper length. In the above, 
% \textcolor{red}{
(1) the equality on the first line uses the determinant of block-diagonal matrix equals to the product of determinant of diagonal blocks, 2) the equality on the last line is due to the matrix determinant lemma, and 3) the inequality on the last line is due to the AM-GM inequality and that $\sum_{i=1}^{B^\prime} n_i = t$. 
% }

Next, since $|\mathcal{P}_t| \ge B^\prime$ and 
$\left( 1 + \frac{s^{-2}t}{x} \right)^x$ is increasing with $x$ (on $ [1, \infty)$), 
\begin{align}
\det \left( s^{-2} K + I \right)  \le \left( 1 + \frac{s^{-2} t}{B^\prime} \right)^{ B^\prime } \le  \left( 1 + \frac{s^{-2} t}{ |\mathcal{P}_t| } \right)^{|\mathcal{P}_t|}. \label{eq:bound-det-by-partition}
\end{align}
% and that $|\mathcal{P}_t| \ge B^\prime$. % that states 
%$$\det \left( A + \bm{u} \bm{v}^\top \right) = \left( 1 + \bm{v}^\top A^{-1} \bm{u} \right) \det \left( A \right).$$
Therefore, from (\ref{eq:recursive-equation}) and (\ref{eq:bound-det-by-partition}), 
\begin{align} 
\sum_{\tau=1}^T \log   \left( 1 + s^{-2} \sigma_{ T, \tau -1}^2 (a_\tau ) \right)  \le | \mathcal{P}_T | \log \left( 1 + \frac{s^{-2} T }{ | \mathcal{P}_T | } \right),  \label{eq:log-sigma-to-partition}
\end{align} 
since arguments after (\ref{eq:entropy-LHS}) hold for all $t \le T$.

Since the function $h(\lambda) = \frac{\lambda}{\log (1 + \lambda)}$ is increasing for non-negative $\lambda$, 
$\lambda \le \frac{s^{-2}_T }{ \log ( 1 + s^{-2}_T ) } \log (1 + \lambda)$
for $\lambda \in [ 0, s^{-2}_T ]$. 
Since $\sigma_{T,t} (a) \in [0, 1]$ for all $a$, 
\begin{align}
    \sigma_{T, t}^2 (a) \le \frac{1}{ \log (1 + s^{-2}_T ) } \log \left( 1 + s^{-2}_T \sigma_{T, t}^2 (a) \right) \label{eq:bound-sigma-by-log-sigma}
\end{align}
for $t, T = 0, 1, 2, \cdots$. 
%Since $\sigma_{T,t} (x) \in [0, 1]$ for all $x$ and $t, T = 0,1,2, \cdots$, and $z \le \frac{s^{-2} }{ \log ( 1 + s^{-2} ) } \log (1 + z)$ for $z \in [ 0, s^{-2} ]$, we know $\sigma_{T, t}^2 (x) \le \frac{1}{ \log (1 + s^{-2}) } \log \left( 1 + s^{-2} \sigma_{T, t}^2 (x) \right)$ for $t, T = 0, 1, 2, \cdots$. 
Since the partitions are nested, we have that for $T_1 \le T_2$, $n_{T_1, t} (a) \ge n_{T_2, t} (a)$, and thus $\sigma_{T_1, t}^2 (a) \le \sigma_{T_2, t}^2 (a)$. 
Suppose we query at points $a_{ 1 }, \cdots, a_T$ in the Gaussian process $\mathcal{GP} (0, k_T ( \cdot, \cdot ) )$. Then, 
% note that arguments after (\ref{eq:entropy-LHS}) holds for all $t \le T$, and get
\begin{align} 
&\quad \sum_{t =  1 }^{T}  \frac{1}{n_{t-1} (a_t)}  \le \sum_{t =  1 }^{T} \frac{1 + s^{-2}_T }{ 1 + s^{-2}_T n_{t-1} (a_t)}  \nonumber \\
&\le \sum_{t =  1 }^{T}  \frac{1 + s^{-2}_T }{ 1 + s^{-2}_T n^0_{T, t-1} (a_t)}  \le  \left( 1 + s^{-2}_T \right) \sum_{t =  1 }^{T} \sigma^2_{T, t-1} (a_t)  \label{eq:use-n-to-sigma} \\
&\le  \frac{ 1 + s^{-2}_T }{ \log ( 1 + s^{-2}_T ) } \sum_{t = 1 }^{T} \log \left( 1 + s^{-2}_T \sigma_{T, t-1}^2 (a_t) \right)  \nonumber \\ 
&\le  \frac{ 1 + s^{-2}_T }{ \log ( 1 + s^{-2}_T ) } | \mathcal{P}_T |  \log \left( 1 + s^{-2}_T \frac{T }{  | \mathcal{P}_T |  } \right), \nonumber % \label{eq:lemma4}
% &\le  l^\prime T^{\frac{d}{d+2}} \log \left(  1 + s^{-2}_T  \frac{ T }{ l^\prime T^{\frac{d}{d+2}}  } \right)  \label{eq:e-increase} 
\end{align} 
where (\ref{eq:use-n-to-sigma}) uses (\ref{eq:n-to-sigma}), the second last inequality uses (\ref{eq:bound-sigma-by-log-sigma}), and the last inequality uses (\ref{eq:log-sigma-to-partition}).
Finally, we optimize over $s_T$. Since $s_T^{-2} = e - 1$ minimizes $\frac{ 1 + s^{-2}_T }{ \log ( 1 + s^{-2}_T ) }$, we have 
\begin{align*}
    \sum_{t =  1 }^{T} \frac{1}{n_{t-1} (a_t)} &\le e | \mathcal{P}_T | \log \left( 1 + ( e - 1 ) \frac{ T }{ | \mathcal{P}_T | } \right). 
\end{align*} 

The above argument proves (\ref{eq:point-scattering-gp}). 

\begin{remark} \label{remark2}
    One important insight of our analysis is that this allows us to link the Hoeffding-type concentration term to the posterior variance of the constructed Gaussian processes. This connection is directly shown in (\ref{eq:n-to-sigma}). As we will discuss in Section \ref{sec:hierarchical-bayesian}, we can use this connection to improve the entire learning process via ``softening''. 
\end{remark} 

Next, we sketch the proofs of (\ref{eq:point-scattering-1}) and (\ref{eq:point-scattering-alpha}). 

\textbf{Proof of (\ref{eq:point-scattering-1}).} 
Consider the partition $ \mathcal{P}_T $ at time $T$. We label the regions of the partitions by $j = 1,2, \cdots, |\mathcal{P}_T|$. Let $t_{j,i}$ be the time when the $i$-th point in the $j$-th region in $\mathcal{P}_T$ being selected. Let $b_j$ be the number of points in region $j$. Since the partitions are nested, we have $1 + n_{t_{j,i} - 1}^0 (x_{t_{j,i}}) \ge i$ for all $i,j$. We have, for $T \ge 1$, 

\begin{align}
    \sum_{t=1}^T\frac{1}{1 + n_{t-1}^0 (x_t)} &= \sum_{j=1}^{|\mathcal{P}_T|}\sum_{i=1}^{b_j} \frac{1}{1 + n_{t_{j,i} - 1}^0 (x_{t_{j,i}}^0 )} \le \sum_{j=1}^{|\mathcal{P}_T|}\sum_{i=1}^{b_j} \frac{1}{i} \label{eq:ineq} \\
    &\le \sum_{j=1}^{|\mathcal{P}_T|} \left( 1 + \log b_j \right)  = |\mathcal{P}_T| +  \sum_{j=1}^{|\mathcal{P}_T|}  \log b_j \nonumber \\ 
    &= |\mathcal{P}_T| + \log  \prod_{j=1}^{|\mathcal{P}_T|} b_j \le |\mathcal{P}_T| + |\mathcal{P}_T| \log  \frac{T}{|\mathcal{P}_T|} \label{eq:am-gm},
\end{align} 
where (\ref{eq:ineq}) uses $1 + n_{t_{j,i} - 1}^0 (x_{t_{j,i}}) \ge i$ and (\ref{eq:am-gm}) uses AM-GM inequality and that $\sum_{j=1}^{|\mathcal{P}_T|} b_j = T$. 

\textbf{Proof of (\ref{eq:point-scattering-alpha}). } 
%\label{app:point-scattering-alpha}
The idea is similar to that of (\ref{eq:point-scattering-1}). 
% The key observation is stated in the first paragraph in Appendix \ref{app:point-scattering-1}. 
For $0< \alpha < 1$, 
\begin{align}
    \sum_{t=1}^T \left( \frac{1}{1 + n_{t-1}^0 (x_t)} \right)^{\alpha} &=  \sum_{j=1}^{|\mathcal{P}_T|}\sum_{i=1}^{b_j} \left( \frac{1}{1 + n_{t_{j,i}}^0 (x_{t_{j,i}}^0 )} \right)^\alpha \nonumber \\
    &\le \sum_{j=1}^{|\mathcal{P}_T|}\sum_{i=1}^{b_j} \frac{1}{i^\alpha} \le  \sum_{j=1}^{|\mathcal{P}_T|}  \frac{1}{1-\alpha} b_j^{1 - \alpha} \nonumber \\
    &\le \frac{1}{1 - \alpha} |\mathcal{P}_T|^\alpha T^{1 - \alpha} \label{eq:holder}, 
\end{align} 
where (\ref{eq:holder}) is due to the H\"older's inequality and that $\sum_{j=1}^{|\mathcal{P}_T|} b_j = T$.

%% file: tex/table-svhn.tex
\begin{table}
\centering
% \vspace{1.5cm} 
\subfloat[][CNN architecture for SVHN. A value with * means that this parameter is tuned, and the batch-normalization layer uses all Tensorflow's default settings.  ]{
\begin{tabular}{ c c c } 
Layer & Hyperparameters & values \\ \hline \hline
\multirow{4}{*}{  Conv1  }  & conv1-kernel-size & *  \\
&  conv1-number-of-channels & 200 \\
&  conv1-stride-size & (1,1) \\
% &  conv1-padding & ``same'' \\ 
\hline
\multirow{3}{*}{  MaxPooling1  }  & pooling1-size & (3,3)  \\
&  pooling1-stride & (1,1) \\
% &  pooling1-padding & ``same'' \\ 
\hline
\multirow{4}{*}{  Conv2  }  & conv2-kernel-size & *  \\
&  conv2-number-of-channels & 200 \\
&  conv2-stride-size & (1,1) \\
% &  conv2-padding & ``same'' \\ 
\hline
\multirow{3}{*}{  MaxPooling2  }  & pooling2-size & (3,3)  \\
&  pooling2-stride & (2,2) \\
% &  pooling2-padding & ``same'' \\ 
\hline
\multirow{4}{*}{  Conv3  }  & conv3-kernel-size & (3,3)  \\
&  conv3-number-of-channels & 200 \\
&  conv3-stride-size & (1,1) \\
% &  conv3-padding & ``same'' \\ 
\hline
\multirow{3}{*}{  AvgPooling3  }  & pooling3-size & (3,3)  \\
&  pooling3-stride & (1,1) \\
% &  pooling3-padding & ``same'' \\ 
\hline
 \multirow{3}{*}{  Dense  }  & batch-normalization & default \\ 
 & number-of-hidden-units & 512 \\ 
 & dropout-rate & 0.5 \\ \hline
\end{tabular}
} \\
% \hfill
% \vspace{1.5cm}
\subfloat[][Hyperparameter search space. $\beta_1$ and $\beta_2$ are parameters for the AdamOptimizer \citep{kingma2014adam}. The learning rate is discretized in the following way: from 1e-6 to is 1 (including the end points), we log-space the learning rate into 50 points, and from 1.08 to 5 (including the end points) we linear-space the learning rate into 49 points. \label{tab:svhn-params}]{
\begin{tabular}{ c c c }
Hyperparameters &  & Range  \\ \hline \hline 
conv1-kernel-size  & \qquad \qquad \qquad & $\{1,2,\cdots,7\}$ \\ \hline 
conv2-kernel-size & & $\{1,2,\cdots,7\}$ \\ \hline 
$\beta_1$ \& $\beta_2$ &  & $\{0 , 0.05 ,  \cdots, 1 \}$  \\ \hline 
% $\beta_2$ & \qquad \qquad & $\{0 , 0.05 ,  \cdots, 1 |$  \\ \hline 
learning-rate & & 1e-6 to 5 \\ \hline 
\makecell{training-iteration } &  & $\{300,400, \cdots,1500\}$ \\ \hline
\end{tabular}
}
\caption{Settings for the SVHN experiments. \label{tab:svhn-arch}}
\end{table}

%% file: tex/table-cifar.tex
\begin{table}
\centering
% \vspace{3cm}
\subfloat[][CNN architecture for CIFAR-10. A value with * means that this parameter is tuned, and the batch-normalization layer uses all Tensorflow's default setting.]{
\begin{tabular}{ c c c } 
Layer & Hyperparameters & values \\ \hline \hline
\multirow{4}{*}{    Conv1  }  &   conv1-kernel-size & *  \\ 
&  conv1-no.-of-channels & 200 \\ 
&  conv1-stride-size & (1,1) \\ 
% &  conv1-padding & ``same'' \\ 
\hline
\multirow{3}{*}{  MaxPooling1  }  & pooling1-size & *  \\
&  pooling1-stride & (1,1) \\
% &  pooling1-padding & ``same'' \\ 
\hline
\multirow{4}{*}{  Conv2  }  & conv2-kernel-size & *  \\
&  conv2-no.-of-channels & 200 \\
&  conv2-stride-size & (1,1) \\
% &  conv2-padding & ``same'' \\ 
\hline
\multirow{3}{*}{  MaxPooling2  }  & pooling2-size & *  \\
&  pooling2-stride & (2,2) \\
% &  pooling2-padding & ``same'' \\ 
\hline
\multirow{4}{*}{  Conv3  }  & conv3-kernel-size & *  \\
&  conv3-no.-of-channels & 200 \\
&  conv3-stride-size & (1,1) \\
% &  conv3-padding & ``same'' \\ 
\hline
\multirow{3}{*}{  AvgPooling3  }  & pooling3-size & *  \\
&  pooling3-stride & (1,1) \\
&  pooling3-padding & ``same'' \\ \hline
 \multirow{3}{*}{  Dense  }  & batch-normalization & default \\ 
 & no.-of-hidden-units & 512 \\ 
 & dropout-rate & 0.5 \\ \hline
\end{tabular}
} \\
\subfloat[][Hyperparameter search space. $\beta_1$ and $\beta_2$ are parameters for the Adamoptimizer. The learning rate is discretized in the following way: from 1e-6 to is 1 (including the end points), we log-space the learning rate into 50 points, and from 1.08 to 5 (including the end points) we linear-space the learning rate into 49 points. The learning-rate-reduction parameter is how many times the learning rate is going to be reduced by a factor of 10.
For example, if the total training iteration is 200, the learning-rate is 1e-6, and the  
learning-rate-reduction is 1, then for the first 100 iteration the learning rate is 1e-6, and the for last 100 iterations the learning rate is 1e-7. \label{tab:cifar-params}]{
\begin{tabular}{ c c c }
Hyperparameters & \quad  & Range  \\ \hline \hline 
conv1-kernel-size &  & $\{1,2,\cdots,7\}$ \\ \hline  
conv2-kernel-size &  & $\{1,2,\cdots,7\}$ \\ \hline  
%& $\{1,2,\cdots,7\}$ \\ \hline 
conv3-kernel-size &   & $\{1,2,3 \}$ \\ \hline 
pooling1-size \& pooling2-size &   & $\{1,2,3\}$ \\ \hline 
%  &   & $\{1,2,3\}$ \\ \hline 
pooling3-size &   & $\{1,2,\cdots, 6 \}$ \\ \hline 
$\beta_1$ \& $\beta_2$ &   & $\{0 , 0.05 ,  \cdots, 1 \}$  \\ \hline 
learning-rate &   & 1e-6 to 5 \\ \hline
learning-rate-redeuction &   & \{1,2,3\} \\ \hline
\makecell{training-iteration } &   & $\{200,400, \cdots,3000\}$ \\ \hline
\end{tabular}
}
%\vspace{0.1cm}
\caption{Settings for CIFAR-10 experiments. 
\label{tab:cifar-arch}}
\end{table} 

%% file: tex/robust.tex
\subsubsection{Hierarchical Bayesian Method for Lipschitz Bandits} \label{sec:hierarchical-bayesian}

Existing Lipschitz bandit algorithms \citep[e.g.,][]{kleinberg2008multi} partition the arm space into disjoint bins. Based on this partition, arms in two different bin do not give information about each other, and all arms within the same bins are viewed as the same. This implicit assumption, however, is obviously untrue. On the other hand, imposing a strong prior on the reward function would break the Lipschitzness assumption. To simultaneously address the above two difficulties, we link the learned tree (or partition) to a Bayesian model in light of our analysis of (\ref{eq:point-scattering-gp}). 
This new viewpoint allows us to ``soften'' the entire model using a hierarchical Bayesian method. 
%To remedy this, we design the following hierarchical Bayesian model in light of 

%Therefore we ask the following questions: 
%\newline
%\newline
%\textit{For a partition-based algorithm, how should we use information across bins? How should we differentiate arms within the same bin?} 
%\newline
%\newline
%Indeed, this question can be answered in light of the Section \ref{app:point-scattering-gp}. We use a tree-fitting procedure to learn a metric that tells us how similar any two points are. This metric simultaneously satisfies: (1) it corresponds to a partition, so that Lipschitzness is respected; and (2) it uses information across bins, and treats arms within the same bin differently. 

%A principled way to achieve this is to smooth out the learned ``tree metric'', as shown in Figure \ref{fig:step}. Apparently, this smoothed metric uses information across different bins. However, arms within the same bin are still treated the same. To handle this, we remain conservative about our learned ``tree metric''. Specifically, we do not use the exact learned partition. Instead, we ``soften'' the boundaries in the partition. 
% In order to do this, we inject a Gaussian noise into the boundaries of the learned partition. 
Formally, at each time $t$, we consider the following hierarchical Bayesian problem with respect to the learned partition $\mathcal{P}_t$. 
Note that this hierarchical Bayesian model is updated whenever we update the partition. This is roughly the same as make a finite partition and treat each bin as an arm, and do not impose extra structures on the reward function. Let $\mathcal{P}_t$ be the learnt partition such that each bin is a rectangle. Then the kernel function is defined as 

\begin{align}
    \widetilde{k}_T (\cdot, \cdot) = \sum_{p \in \mathcal{P}_T} \widetilde{k}_T^{(p)} (\cdot, \cdot), 
\end{align}
where $ p $ are regions in $\mathcal{P}_T$, and  $\widetilde{k}_T^{(p)} (\cdot, \cdot)$ is defined as follows. For a partition $ p = \prod_{i=1}^d [a_i, b_i] $, define
\begin{align}
    &\widetilde{k}_T^{(p)} (\cdot, \cdot) = \prod_{i=1}^d \widetilde{k}_T^{(p, i)} (\cdot, \cdot) , \quad \text{where} \label{eq:soft-k-1} \\
    &\widetilde{k}_T^{(p, i)} ( \bm{x} , \bm{x}') = 
    \left[ 1 + \exp \left( - \alpha_T \left( \Delta_i - \frac{b_i - a_i}{2} \right) \right) \right]^{-1}, \label{eq:soft-k-2} \\
    &\Delta_i = \max \left\{ \left| \bm{x}_i - \frac{a_i + b_i}{2} \right|, \left| \bm{x}_i' - \frac{a_i + b_i}{2} \right| \right\} ,  \label{eq:soft-k-3}
\end{align}
where $\bm{x}_i$ (resp. $\bm{x}_i'$) are the $i$-th entry of $\bm{x}_i$ (resp. $\bm{x}'$), and $\alpha_T > 0$ are parameters that controls how smooth are the smoothed tree metrics. 
Given a learned partition $\mathcal{P}_T = \{ p_1, p_2, \cdots, p_K \}$, where $p_{j} = \prod_{i=1}^d [ a_{j}^{(i)}, b_{j}^{(i)}]$, we construct the following hierarchical Bayesian model
\begin{align} 
    \widetilde{a}_{j}^{(i)} &\sim \mathcal{ N } ({a}_{j}^{(i)}, \sigma^2 ), \text{ for all } i,j;  \quad \widetilde{b}_{j}^{(i)} \sim \mathcal{ N } ({b}_{j}^{(i)}, \sigma^2 ), \text{ for all } i,j \\
    \widetilde{k}_T &= \sum_{j = 1}^K  \widetilde{k}_T^{(p_j)} , \text{ where $\widetilde{k}_T^{(p_j)}$ is defined respect to $\prod_{i=1}^d [\widetilde{a}_{j}^{(i)}, \widetilde{b}_{j}^{(i)}]$ } \nonumber \\ 
    f &\sim \mathcal{GP} \left( 0, \widetilde{k}_T (\cdot, \cdot) \right) \\ 
    y &= f + \epsilon, \quad \text{where } \epsilon \sim \mathcal{N} ( 0, \sigma_y^2) .
\end{align}

\begin{figure}
    \centering
    \includegraphics[scale = 0.37]{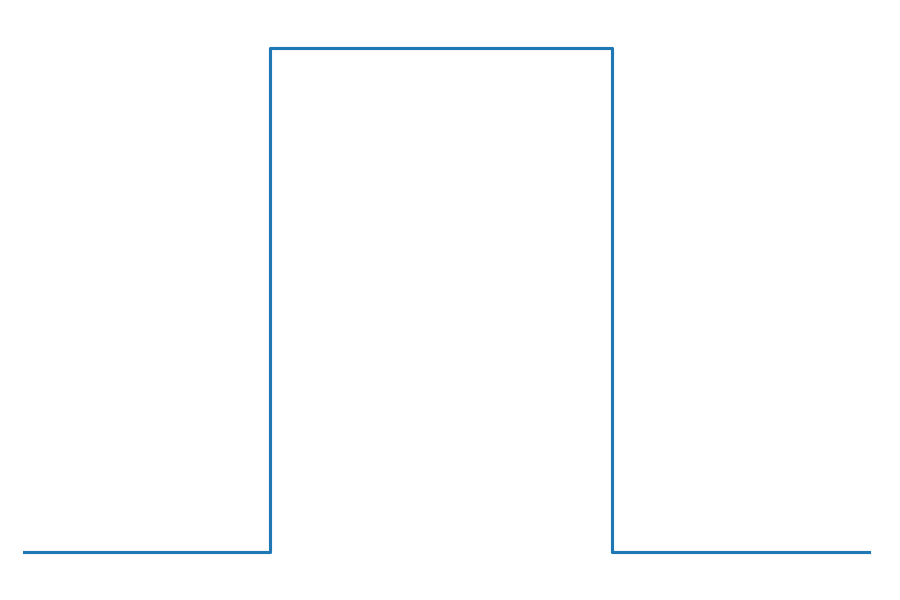} \includegraphics[scale = 0.37]{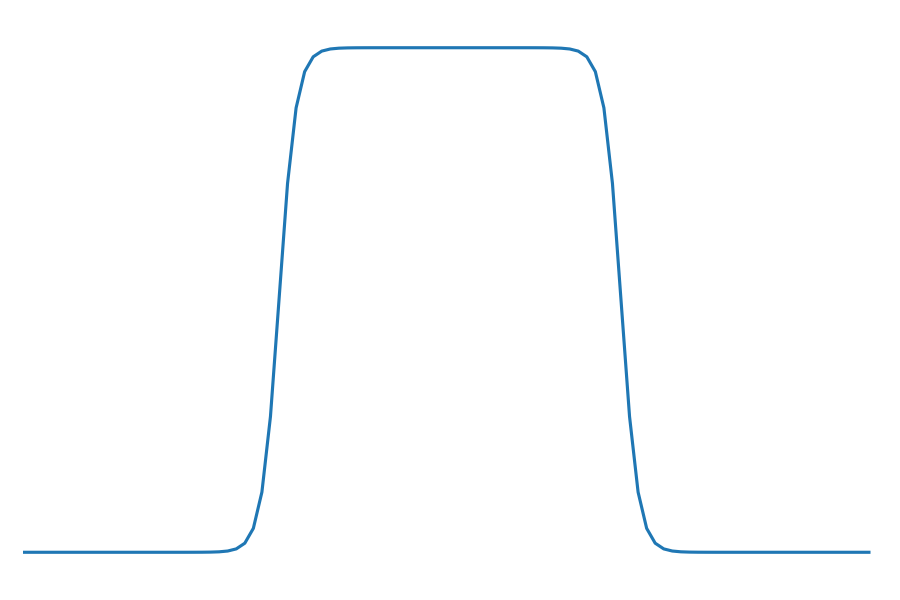}
    \caption{The left subfigure is the metric learned by Algorithm \ref{alg:tucb} (\ref{eq:tree-kernel}). The right subfigure is the smoothed version of this learned metric. \label{fig:step}}
\end{figure}

This hierarchical model has several advantages: (1) It respect Lipschitzness. As we collect more observations, the partition can grow arbitrarily fine, and the approximation can be arbitrarily close to an extract indicator function. Because of this, the no prior smoothness assumption on the true (unknown) reward function is needed. (2) It treats arms within the same bin differently, and can use information across bins. 
% This is especially useful when a bin has few observations: it allows use information from other bins, and (2) a bin a very large: it differentiates arms within the same bin. 

Going back to bandit learning process, we can replace the mean and/or confidence intervals of UCB index with the posteriors of this hierarchical bayesian model. As we discussed in Remark \ref{remark2}, a key insight of our analysis is the link between the Hoeffding-type concentration interval to the posterior variance of the Gaussian processes, which allows us to do this principled substitute. 
In Section \ref{subsec:exp-gp}, we empirically study this hierarchical Bayesian model.

%Surprisingly, the posterior of this long hierarchical model can be approximated inferenced fairly intuitively. 
%
%To construct an intuitive approximation, we note (1) within each ``smoothed partition'' (right sub-figure of Figure \ref{fig:step}), points that are closer to the center of the region are given larger weights, whereas points that are closer to the boundary of the regions are given smaller weights. 
%(2) Since the boundaries of the regions are softened with a prior, effectively this can be viewed as noise corrupts the input data. From here, we can apply results from Gaussian processes with noise in the input space \citep{mchutchon2011gaussian} (Theorem \ref{thm:gp} below). 
%
%\begin{theorem}[\citep{mchutchon2011gaussian}] \label{thm:gp}
%	
%\end{theorem}
%A combination of the above two items allows us to use the following approximated posterior . 
%To verify this approximated inference, we provide synthetic studies in the experiments section. 

%To see how we approximate this posterior, we again consider the example in Figure \ref{fig:step}. To start with, we consider a model where the smoothed metric is used, and there is no randomness in the partition boundaries. 

% This metric effectively puts more weights on points that are closer to the center of the . 

% \textcolor{red}{write a more formal derivation.}

%In light of the point scattering inequalities: 
%Even if measures, this problem persists. 

%% file: tex/exp-new.tex
%\vspace{-0.5cm}
\section{Empirical Study}
%\vspace{-0.4cm}
\label{sec:exp} 
Since the TreeUCB algorithm imposes only mild constraints on tree formation, we use greedy decision tree splitting to fit the reward function,
% instead of controlling only the region diameters and the partition cardinality. 
% Next, we sketch the tree fitting rule we use in our experiments. \textcolor{red}{this part is slightly repetitive with the next part - I would just combine them.}
% \textcolor{blue}{combine?}
% \textbf{Tree Fitting Rule}. 
using the following splitting rule: we find the split that maximizes the reduction in the Mean Absolute Error (MAE), and we stop growing the tree once the maximal possible reduction is below 0.001. 

\subsection{Gaussian Processes with Learned Kernel} 
\label{subsec:exp-gp}

In this section, we compare several baselines, including piecewise constant estimates (within each bin), a Gaussian process regression with box kernel (left subfigure in Figure \ref{fig:step}) and Gaussian process regression with softened box kernel (right subfigure in Figure \ref{fig:step}). The splitting procedure is the same for all methods, so the partitions are the same for the methods. Our results, shown in Figure \ref{fig:smooth-GP}, demonstrates a transition from the hardness of the piecewise constant estimate to the softness of the Gaussian process regression with the softened kernel. This justifies the ``softening'' discussed in Section \ref{sec:hierarchical-bayesian}. 
The Gaussian process kernel parameters for $GP_{S,1}, GP_{S,2}, GP_{S,3}, GP_{S,4}, GP_{S,5}$, namely $\alpha_T$ in Eq. (\ref{eq:soft-k-2}), were set to $ 10,50,100,500,1000 $ respectively.

\begin{figure}[t]
    \centering
	\includegraphics[scale = 0.54]{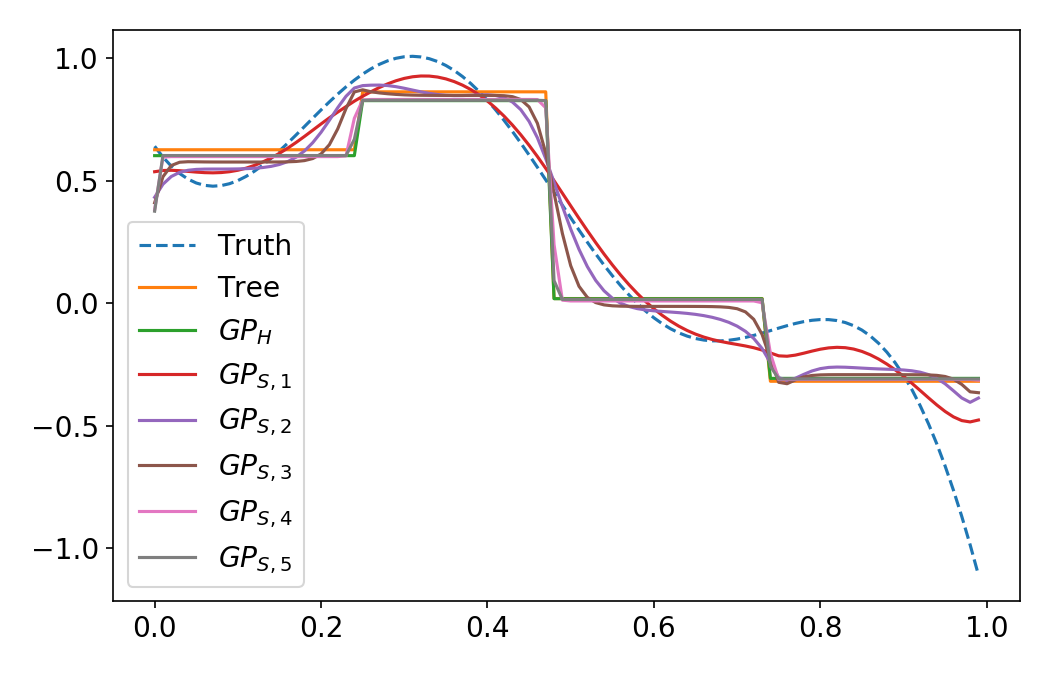}
	\caption{The estimates for a function with respect to a given partition. The ``Tree'' line is directly averaging within each partition. The ``$GP_H$'' line is the learned posterior GP mean function using the ``hard metric.'' The lines ``$GP_{S,1}$ - $GP_{S,5}$'' are 5 learned posterior GP mean functions using the ``soft metric'' (Eq. (\ref{eq:soft-k-1}) - (\ref{eq:soft-k-3})). \label{fig:smooth-GP}} 
	
% 	\begin{minipage}[c]{0.32\textwidth}
%         \includegraphics[width = \textwidth]{./figures/smoothed_GP.png}
%     \end{minipage}\hfill
%     \begin{minipage}[c]{0.12\textwidth}
%         \caption{The estimates for a function with respect to a given partition. The ``Tree'' line is directly averaging within each partition. The ``$GP_H$'' line is the learned posterior GP mean function using the ``hard metric.'' The lines ``$GP_{S,1}$ - $GP_{S,5}$'' are 5 learned posterior GP mean functions using the ``soft metric'' (Eq. (\ref{eq:soft-k-1}) - (\ref{eq:soft-k-3})). \label{fig:smooth-GP}} 
%     \end{minipage}
\end{figure}

\subsection{Application to Neural Network Tuning}
\label{sec:tune-nn} 
%\vspace{-0.2cm} 
One application of stochastic bandit algorithms is zeroth order optimization. In this section, we apply TUCB to tuning neural networks. In this setting, we treat the hyperparameter configurations (e.g., learning rate, network architecture) as the arms of the bandit, and use validation accuracy as reward. The task is to select a hyperparameter configuration and train the network to observe the validation accuracies, and find the best hyperparameter configuration rapidly. This experiment shows that TUCB can compete with the state-of-the-art tuning methods on such hard real-world tasks. %These plots also show TUCB's scalability when implemented using  sci-kit-learn. 

The architecture and the hyperparameter space for the simple Multi-Layer Perceptron (MLP) for the MNIST dataset are: in the feed-forward direction, there are the \textit{input layer}, the \textit{fully connected hidden layer with dropout ensemble}, and then the \textit{output layer}. The hyperparameter search space is five dimensional, including \textit{number of hidden neurons} (range $[10, 784 ]$), \textit{learning rate}  ($[0.0001, 4 )$), \textit{dropout rate} ($[0.1, 0.9)$), \textit{batch size} ($[10, 500]$), and \textit{number of iterations} ($[30, 243]$). 
% Note that for this experiment, the network architecture .  the time is limited to 2 minutes, which, in most cases, does not permit over 99\% acurracy for any methods. 

The details of the CNN setting for SVHN and CIFAR-10 can be found in Tables \ref{tab:svhn-arch} and \ref{tab:cifar-arch}. The results are found in Figure \ref{fig:tune-nn}, indicating that  
TUCB outperforms existing state-of-the-art software packages for tuning neural network methods. 

\begin{figure*}
\centering
\includegraphics[width = \textwidth]{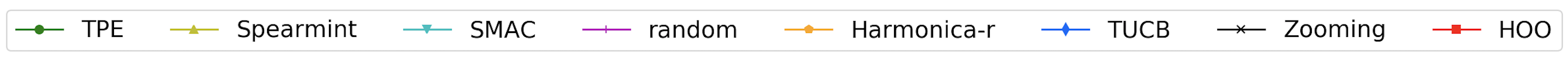} \\
%\vspace{-0.4cm}
\subfloat[MLP for MNIST]{\includegraphics[scale = 0.36]{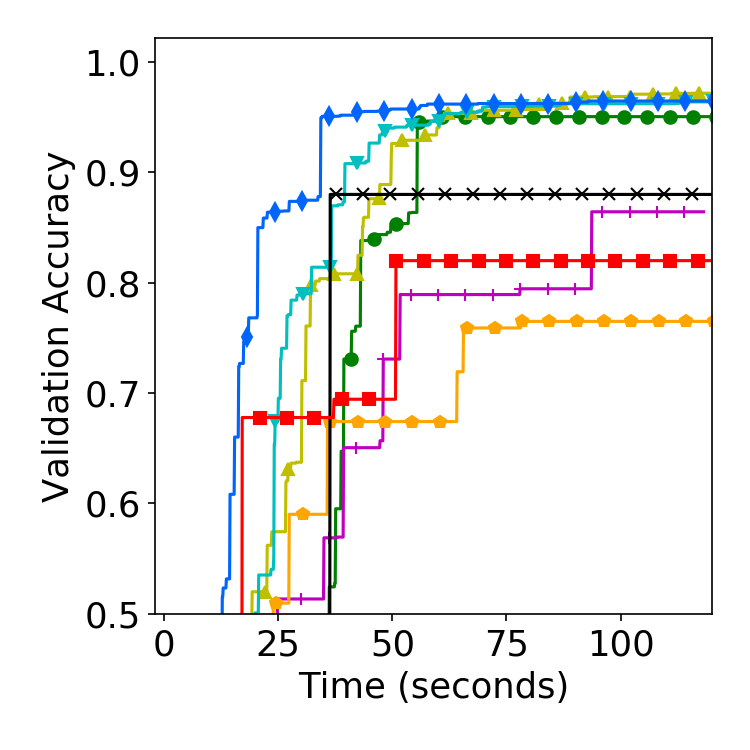} \label{fig:tune-nn-mnist}}
\subfloat[CNN for SVHN]{ \includegraphics[scale = 0.36]{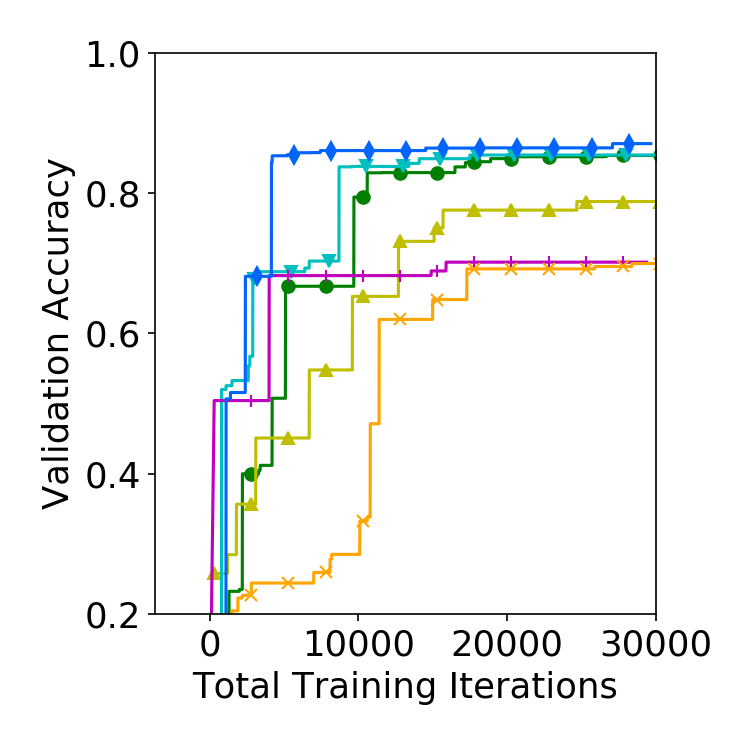} \label{fig:tune-nn-svhn}} 
\subfloat[CNN for CIFAR-10]{ \includegraphics[scale = 0.36]{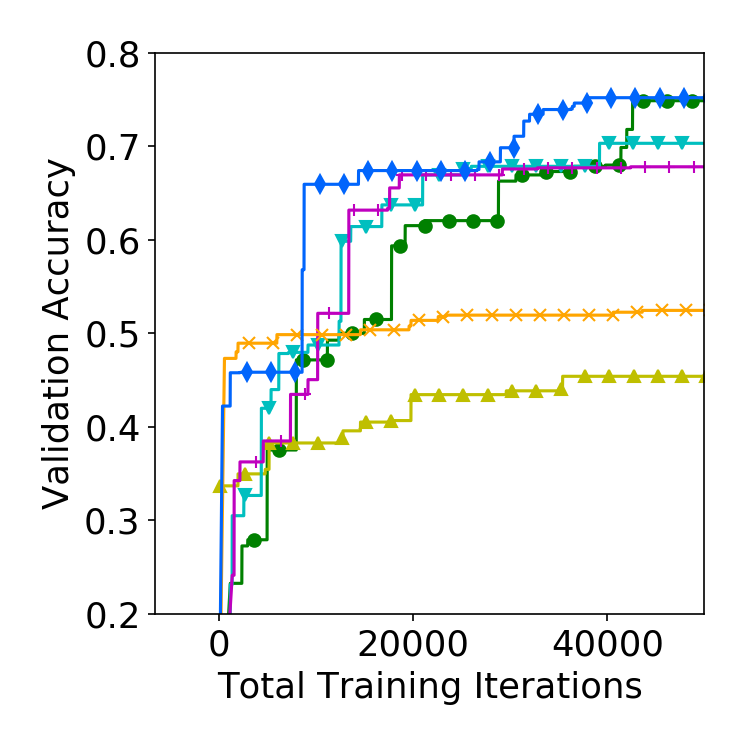}  \label{fig:tune-nn-cifar}}
\caption{
% Performance of TUCB against benchmark methods in tuning neural networks. 
For MNIST, each plot is averaged over 10 runs. For SVHN and CIFAR-10, each plot is averaged over 5 runs. The implementation of TUCB here uses the scikit-learn package \citep{scikitlearn}. In the left-most subplot, x-axis is time (in seconds). This shows TUCB's scalability, since TUCB's curve goes up the fastest. In (a), we use clock time as cost measure.  \label{fig:tune-nn}}
%\vspace{-0.6cm}
\end{figure*}

% In Figure \ref{fig:tune-nn}, the results shows that TUCB is competitive with benchmark methods in the task of tuning neural networks.
% See Appendix \ref{app:exp-nn} for details about hyperparameter space for this experiments. 

%% file: tex/conclusion.tex
% !TEX root = ../writeup.tex
% \vspace{-10pt}
%\vspace{-0.5cm}
\section{Conclusion}
%\vspace{-0.4cm}
We propose the TreeUCB and the Contextual TreeUCB frameworks that use decision trees (regression trees) to flexibly partition the arm space and the context-arm space as an Upper Confidence Bound strategy is played across the partition regions. We also provide regret analysis via the point scattering inequalities. We provide implementations using decision trees that learn the partition. 
% optimizes \textcolor{red}{You are not learning the reward function. The reward function is fixed. You are optimizing it, not learning it - you'd learn a model but not a reward function unless you're doing reinforcement learning or something} the reward function. 
% Empirical studies show that %\textcolor{red}{
% TUCB and CTUCB can empirically outperform HOO, Zooming bandit, and CCKL-UCB. 
TUCB is competitive with the state-of-the-art hyperparameter optimization methods in hard tasks like neural-net tuning, and could save substantial computing resources. 
This suggests that, in addition to random search and Bayesian optimization methods, more bandit algorithms should be considered as benchmarks for difficult real-world problems such as neural network tuning. 
% and other difficult optimization problems involving arms (variables).
% that might be either discrete or continuous, or a mixture of both. 
%say something here about bandit algorithms should be part of the benchmarks for NN tuning.
%} %In addition to the performance improvement, decision trees can be generalized to forests, which are scalable and expressive. Therefore, applying TUCB and CTUCB to data-intensive problems is a future direction of study. 
%One future direction is to incorporate expert information/contextual information when making decisions. 

%%%%%%% No acknowledgments for submission!
\section*{Acknowledgement}

The authors are grateful to Aaron J Fisher and Tiancheng Liu for insightful discussions. The authors thank anonymous reviewers for valuable feedback. The project is partially supported by the Alfred P. Sloan Foundation through the Duke Energy Data Analytics fellowship. 